\documentclass{article}

\usepackage[utf8]{inputenc}
\usepackage[T1]{fontenc}
\usepackage{hyperref}
\usepackage{url}
\usepackage{booktabs}
\usepackage{amsfonts}
\usepackage{nicefrac}
\usepackage{microtype}

\usepackage[small]{caption}
\usepackage{graphicx}
\usepackage{algorithm}
\usepackage{algorithmic}
\usepackage{enumitem}

\usepackage{amsthm,amsmath,amssymb}

\usepackage{natbib}

\usepackage{pgfplots}
\usepgfplotslibrary{fillbetween}
\usetikzlibrary{positioning}

\usepackage{subfigure}
\usepackage{wrapfig}

\usepackage[verbose=true,letterpaper]{geometry}
\AtBeginDocument{
  \newgeometry{
    textheight=9in,
    textwidth=6in,
    top=1in,
    headheight=12pt,
    headsep=25pt,
    footskip=30pt
  }
}

\theoremstyle{definition}
\newtheorem{theorem}{Theorem}
\newtheorem{lemma}[theorem]{Lemma}

\newtheorem{definition}[theorem]{Definition}
\newtheorem{problem}[theorem]{Problem}

\newcommand{\argmin}{\operatornamewithlimits{argmin}}

\usepackage{authblk}

\title{Data Cleansing for Models Trained with SGD}

\author[1]{Satoshi Hara\thanks{\texttt{satohara@ar.sanken.osaka-u.ac.jp}}}
\author[2]{Atsushi Nitanda\thanks{\texttt{nitanda@mist.i.u-tokyo.ac.jp}}}
\author[3]{Takanori Maehara\thanks{\texttt{takanori.maehara@riken.jp}}}
\affil[1]{Osaka University, Japan}
\affil[2]{The University of Tokyo, Japan}
\affil[3]{RIKEN AIP, Japan}
\date{}

\begin{document}

\maketitle

%%%%%%%%%%%%%%%%%%%%%%%%%%%%%%%%%%%%%%%%%%%%%%%%%%%%%%%%%%%%%%%%%%%%
\begin{abstract}
Data cleansing is a typical approach used to improve the accuracy of machine learning models, which, however, requires extensive domain knowledge to identify the influential instances that affect the models.
In this paper, we propose an algorithm that can suggest influential instances without using any domain knowledge.
With the proposed method, users only need to inspect the instances suggested by the algorithm, implying that users do not need extensive knowledge for this procedure, which enables even non-experts to conduct data cleansing and improve the model.
The existing methods require the loss function to be convex and an optimal model to be obtained, which is not always the case in modern machine learning.
To overcome these limitations, we propose a novel approach specifically designed for the models trained with stochastic gradient descent (SGD).
The proposed method infers the influential instances by retracing the steps of the SGD while incorporating intermediate models computed in each step.
Through experiments, we demonstrate that the proposed method can accurately infer the influential instances.
Moreover, we used MNIST and CIFAR10 to show that the models can be effectively improved by removing the influential instances suggested by the proposed method.
\end{abstract}

%%%%%%%%%%%%%%%%%%%%%%%%%%%%%%%%%%%%%%%%%%%%%%%%%%%%%%%%%%%%%%%%%%%%
\section{Introduction}

Building accurate models is one of the fundamental goals in machine learning.
If the obtained model is not satisfactory, users try to improve the model in several ways such as by modifying input features, cleansing data, or even by gathering additional data.
Error analysis~\citep{ng2017machine} is a typical approach for this purpose.
In this analysis, the users hypothesize the cause of model's failure by investigating important features or examining the misclassified instances.
However, a good hypothesis requires experience and domain knowledge.
Therefore, it is difficult for non-domain experts or non-machine learning specialists to build accurate models.

How can we help non-experts to build accurate machine learning models?
In this study, we focus on the following data cleansing problem that removes ``harmful'' instances from the training set.
\begin{problem}[Data Cleansing]
Find a subset of the training instances such that the trained model obtained after removing the subset has a better accuracy.
\end{problem}
Currently, the users hypothesize the training instances that can have certain influences on the resulting models by inspecting instances based on the domain knowledge.
Our aim is to develop an algorithm that suggests influential instances \emph{without using any domain knowledge}.
With such an algorithm, the users do not need to hypothesize influential instances.
Instead, they only need to inspect the instances suggested by the algorithm.
If some of the suggested instances are deemed to be inappropriate, such as the instances that are irrelevant to the targeting tasks, the users can merely remove them.
Hence, with this process, even non-experts can improve the models.

For data cleansing, we need to determine the training instances that affect the model.
In the literature of statistics, an \emph{influential instance} is defined as the instance that leads to a distinct model from the current model if the corresponding instance is absent~\citep{cook1977detection}.
A naive approach to determine these influential instances is, therefore, to retrain the model by leaving every one instance out of the training set, which can be computationally very demanding.
To efficiently infer an influential instance without retraining, the convexity of the loss function plays an important role.
Pioneering studies by \cite{beckman1974distribution}, \cite{cook1977detection}, and \cite{pregibon1981logistic} have shown that, for some convex loss functions, the influential instances can be inferred without model retraining by utilizing the optimality condition on the training loss, given that an optimal model is obtained.
A recent study by \cite{koh2017understanding} further generalized these approaches to any smooth and strongly convex loss functions by incorporating the idea of influence function~\citep{cook1980characterizations} in robust statistics.

The focus of this study is to go beyond the convexity and optimality.
We aim to develop an algorithm that can infer influential instances even for non-convex objectives such as deep neural networks.
To this end, we propose a completely different approach to infer the influential instances.
The proposed approach is based on the stochastic gradient descent (SGD).
Modern machine learning models including deep neural networks are trained using SGD and its variants.
Our idea is to redefine the notion of influence for the models trained with SGD, which we named \emph{SGD-influence}.
Based on SGD-influence, we propose a method that infers the influential instances without model retraining.
The proposed method is based solely on the analysis of SGD.
Different from the existing methods, the proposed method does not require the optimality conditions to hold true on the obtained models.
The proposed method is therefore suitable to the SGD context where we no longer look for the exact optimum of the training loss.
In SGD, we instead look for the minimum error on the validation set, which leads to early stopping of the optimization that can violate the optimality condition.

In summary, the contribution of this study is threefold.
\begin{itemize}
	\item We propose a new definition of the influence, which we name as \emph{SGD-influence}, for the models trained with SGD. SGD-influence is defined based on the counterfactual effect: what if an instance is absent in SGD, how largely will the resulting model change?
	\item We propose a novel estimator of SGD-influence based on the analysis of SGD. We then construct a proposed influence estimation algorithm based on this estimator. We also study the estimation error of the proposed estimator on both convex and non-convex loss functions.
	\item Through experiments, we demonstrate that the proposed method can accurately infer the influential instances. Moreover, we used MNIST and CIFAR10 to show that the models can be effectively improved by removing the influential instances suggested by the proposed method.
\end{itemize}

\paragraph{Notations}
For vectors $a, b \in \mathbb{R}^p$, we denote the inner product by $\langle a, b \rangle = \sum_{i=1}^p a_i b_i$, and the norm by $\|a\| = \sqrt{\langle a, a \rangle}$.
For a function $f(\theta)$ with $\theta \in \mathbb{R}^p$, we denote its derivative by $\nabla_\theta f(\theta)$.

%%%%%%%%%%%%%%%%%%%%%%%%%%%%%%%%%%%%%%%%%%%%%%%%%%%%%%%%%%%%%%%%%%%%
\section{Preliminaries}

Let $z = (x, y) \in \mathbb{R}^d \times \mathcal{Y}$ be an observation, which is a pair of $d$-dimensional input feature vector $x$ and output $y$ in a certain domain $\mathcal{Y}$ (e.g., $\mathcal{Y} = \mathbb{R}$ for regression, and $\mathcal{Y} = \{-1, 1\}$ for binary classification).
The objective of learning is to find a model $f(x; \theta)$ that well approximates the output as $y \approx f(x; \theta)$.
Here, $\theta \in \mathbb{R}^p$ is a parameter of the model.

Let $D := \{z_n = (x_n, y_n)\}_{n=1}^N$ be a training set with independent and identically distributed observations.
We denote the loss function for an instance $z$ with the parameter $\theta$ by $\ell(z; \theta)$.
The learning problem is then denoted as
\begin{align}
    \hat{\theta} = \argmin_{\theta \in \mathbb{R}^p} \frac{1}{N} \sum_{n=1}^N \ell(z_n; \theta) .
    \label{eq:obj}
\end{align}

\paragraph{SGD}
Let $g(z; \theta) := \nabla_\theta \ell(z; \theta)$.
SGD starts the optimization from the initial parameter $\theta^{[1]}$.
An update rule of the mini-batch SGD at the $t$-th step for the problem~(\ref{eq:obj}) is given by $\theta^{[t+1]} \leftarrow \theta^{[t]} -  \frac{\eta_t}{|S_t|} \sum_{i \in S_t} g(z_i; \theta^{[t]})$, where $S_t$ denotes the set of instance indices used in the $t$-th step, and $\eta_t > 0$ is the learning rate.
We denote the number of total SGD steps by $T$.

%%%%%%%%%%%%%%%%%%%%%%%%%%%%%%%%%%%%%%%%%%%%%%%%%%%%%%%%%%%%%%%%%%%%
\section{SGD-Influence}

We propose a novel notion of influence for the models trained with SGD, which we name as \emph{SGD-influence}.
We then formalize the influence estimation problem we consider in this paper.

We define SGD-influence based on the following \emph{counterfactual} SGD where one instance is absent.
\begin{definition}[Counterfactual SGD]
The counterfactual SGD starts the optimization from the same initial parameter as the ordinary SGD $\theta_{-j}^{[1]} = \theta^{[1]}$.
The $t$-th step of the counterfactual SGD with the $j$-th instance $z_j$ absent is defined by $\theta_{-j}^{[t+1]} \leftarrow \theta_{-j}^{[t]} - \frac{\eta_t}{|S_t|} \sum_{i \in S_t \setminus \{j\}} g(z_i; \theta_{-j}^{[t]})$.
\end{definition}
\begin{definition}[SGD-Influence]
We refer to the parameter difference $\theta_{-j}^{[t]} - \theta^{[t]}$ as the \emph{SGD-influence} of the instance $z_j \in D$ at step $t$.
\end{definition}
It should be noted that SGD-influence can be defined in every step of SGD, even for non-optimal models.
Thus, SGD-influence is a suitable notion of influence for the cases where we no longer look for the exact optimal of (\ref{eq:obj}).
In this study, we specifically focus on estimating an inner product of a query vector $u \in \mathbb{R}^p$ and the SGD-influence after $T$ SGD steps, as follows.
\begin{problem}[Linear Influence Estimation (LIE)]
For a given query vector $u \in \mathbb{R}^p$, estimate the \emph{linear influence} $L_{-j}^{[T]}(u) := \langle u, \theta_{-j}^{[T]} - \theta^{[T]} \rangle$.
\end{problem}
LIE includes several important applications~\citep{koh2017understanding}.
If we take $u$ as the derivative of the prediction function $f$ for an input $x$, i.e., $u = \nabla_{\theta} f(x; \theta^{[T]})$, LIE amounts to estimating the change of the predicted output for $x$, based on the first-order Taylor approximation $L_{-j}^{[T]}(\nabla_{\theta} f(x; \theta)) \approx f(x; \theta_{-j}^{[T]}) - f(x; \theta^{[T]})$.
For example, in logistic regression, by taking $f$ as the logit function, positive linear influence $L_{-j}^{[T]}(\nabla_{\theta} f(x; \theta^{[T]}))$ for the negatively predicted input $x$ indicates that the prediction can be changed to be positive by removing the instance $z_j$.

Another important application is the influence estimation on the loss.
If we take $u = \nabla_{\theta} \ell(x; \theta^{[T]})$ for an input $x$, LIE amounts to estimating the change in loss $L_{-j}^{[T]}(\nabla_{\theta} \ell(x; \theta^{[T]})) \approx \ell(x; \theta_{-j}^{[T]}) - \ell(x; \theta^{[T]})$.
Negative $L_{-j}^{[T]}(\nabla_{\theta} \ell(x; \theta^{[T]}))$ indicates that the loss on the input $x$ can be decreased by removing $z_j$.

Note that SGD-influence as well as linear influence can be computed exactly by running the counterfactual SGD for all $z_j \in D$.
However, this requires running SGD $N$ times, which is computationally demanding even for $N \approx 100$.
Therefore, our goal is to develop an estimation algorithm for LIE, which does not require running SGD multiple times.

%%%%%%%%%%%%%%%%%%%%%%%%%%%%%%%%%%%%%%%%%%%%%%%%%%%%%%%%%%%%%%%%%%%%
\section{Estimating SGD-Influence}
\label{sec:est}

In this section, we present our proposed estimator of SGD-influence and show its theoretical properties.
We then derive an algorithm for LIE based on the estimator in the next section.

%%%%%%%%%%%%%%%%%%%%%%%%%%%%%%%%%%%%%%%%%%%%%%%%%%%%%%%%%%%%%%%%%%%%
\subsection{Proposed Estimator}

We estimate SGD-influence using the first-order Taylor approximation of the gradient.
Here, we assume that the loss function $\ell(z;\theta)$ is twice differentiable.
We then obtain $\frac{1}{|S_t|}\sum_{i \in S_t} \left(\nabla_\theta \ell(z_i; \theta_{-j}^{[t]}) - \nabla_\theta \ell(z_i; \theta^{[t]})\right) \approx H^{[t]} (\theta_{-j}^{[t]} - \theta^{[t]})$, where $H^{[t]} := \frac{1}{|S_t|} \sum_{i \in S_t} \nabla_\theta^2 \ell(z_i; \theta^{[t]})$ is the Hessian of the loss on the mini-batch $S_t$.
With this approximation, denoting an identity matrix by $I$, we have
\begin{align*}
    \theta_{-j}^{[t]} - \theta^{[t]} & = (\theta_{-j}^{[t-1]} - \theta^{[t-1]}) - \frac{\eta_{t-1}}{|S_{t-1}|} \sum_{i \in S_{t-1}} (\nabla_\theta \ell(z_i; \theta_{-j}^{[t-1]}) - \nabla_\theta \ell(z_i; \theta^{[t-1]})) \\
    & \approx (I - \eta_{t-1} H^{[t-1]})  (\theta_{-j}^{[t-1]} - \theta^{[t-1]}) .
\end{align*}

We construct an estimator for the SGD-influence based on this approximation.
For simplicity, here, we focus on one-epoch SGD where each instance appears only once.
Let $Z_t := I - \eta_t H^{[t]}$ and $\pi(j)$ be the SGD step where the instance $z_j$ is used.
By recursively applying the approximation and recalling that $\theta_{-j}^{[\pi(j)+1]} - \theta^{[\pi(j)+1]} = \frac{\eta_{\pi(j)}}{|S_{\pi(j)}|} g(z_j; \theta^{[\pi(j)]})$, we obtain the following estimator
\begin{align}
    \theta_{-j}^{[T]} - \theta^{[T]} \approx \frac{\eta_{\pi(j)}}{|S_{\pi(j)}|} Z_{T-1}  Z_{T-2} \cdots Z_{\pi(j)+1} g(z_j; \theta^{[\pi(j)]}) =: \Delta \theta_{-j} .
    \label{eq:diff_approx}
\end{align}

%%%%%%%%%%%%%%%%%%%%%%%%%%%%%%%%%%%%%%%%%%%%%%%%%%%%%%%%%%%%%%%%%%%%
\subsection{Properties of $\Delta \theta_{-j}$}

Here, we evaluate the estimation error of the proposed estimator $\Delta \theta_{-j}$ for both convex and non-convex loss functions.
A notable property of the estimator $\Delta \theta_{-j}$ is that, unlike existing methods, the error can be evaluated \emph{even without assuming the convexity of the loss function $\ell(z; \theta)$}.

\paragraph{Convex Loss}
For smooth and strongly convex problems, there exists a uniform bound on the gap between the SGD-influence $\theta_{-j}^{[T]} - \theta^{[T]}$ and the proposed estimator $\Delta \theta_{-j}$.
\begin{theorem}
\label{th:convex}
Assume that $\ell(z; \theta)$ is twice differentiable with respect to the parameter $\theta$ and there exist $\lambda, \Lambda > 0$ such that $\lambda I \prec \nabla^2_\theta \ell(z; \theta) \prec \Lambda I$ for all $z, \theta$.
If $\eta_s \leq 1/\Lambda$, then we get
\begin{align}
    \|(\theta_{-j}^{[T]} - \theta^{[T]}) - \Delta \theta_{-j}\| \le \sqrt{2 (h_j(\lambda)^2 + h_j(\Lambda)^2)} ,
    \label{eq:bound_convex}
\end{align}
where $h_j(a) := \frac{\eta_{\pi(j)}}{|S_{\pi(j)}|}\prod_{s=\pi(j)+1}^{T-1}(1-\eta_{s} a) \| g(z_j; \theta^{[\pi(j)]}) \|$.
\end{theorem}

\paragraph{Non-Convex Loss}
For non-convex loss functions, the aforementioned uniform bound no longer holds.
However, we can still evaluate the growth of the estimation error.
For simplicity, we consider a constant learning rate $\eta=O(\gamma/\sqrt{T})$ that depends only on the number of total SGD steps $T$.
It should be noted that SGD with this learning rate is theoretically justified to converge to a stationary point \citep{ghadimi2013stochastic}.
The next theorem indicates that $\Delta \theta_{-j}$ can approximate SGD-influence well if Hessian $\nabla_\theta^2 \ell(\theta,z)$ is Lipschitz continuous.

\begin{theorem}
\label{th:nonconvex}
Assume that $\ell(z; \theta)$ is twice differentiable and 
$\nabla_\theta^2 \ell(z; \theta)$ is $L$-Lipschitz continuous with respect to $\theta$.
Moreover, assume that $\| \nabla_\theta \ell(z; \theta) \| \leq G$, $\nabla^2_\theta \ell(z; \theta) \prec \Lambda I$ for all $z, \theta$.
Consider SGD with a learning rate $\eta=O(\gamma/\sqrt{T})$.
Then,
\begin{align}
    \|(\theta_{-j}^{[T]} - \theta^{[T]}) - \Delta \theta_{-j}\| \leq \frac{\exp^{ O(\gamma \Lambda \sqrt{T}) } \gamma^{2} T G^{2} L}{\Lambda}.
    \label{eq:bound_nonconvex}
\end{align}
\end{theorem}

%%%%%%%%%%%%%%%%%%%%%%%%%%%%%%%%%%%%%%%%%%%%%%%%%%%%%%%%%%%%%%%%%%%%
\section{Proposed Method for LIE}
\label{sec:lie}

We now derive our proposed method for LIE.
First, we extend the estimator $\Delta \theta_{-j}$ to multi-epoch SGD.
Let $\pi_1(j), \pi_2(j), \ldots, \pi_K(j)$ be the steps where the instance $z_j$ is used in $K$-epoch SGD.
We estimate the effect of the step $\pi_k(j)$ based on (\ref{eq:diff_approx}) as $Z_{T-1} Z_{T-2} \cdots Z_{\pi_k(j)+1} \frac{\eta_{\pi_k(j)}}{|S_{\pi_k(j)}|} g(z_j; \theta^{[\pi_k(j)]})$.
We then add all the effects and derive the estimator $\Delta \theta_{-j} = \sum_{k=1}^K \left(\prod_{s=1}^{T-\pi_k(j)-1} Z_{T-s}\right) \frac{\eta_{\pi_k(j)}}{|S_{\pi_k(j)}|} g(z_j; \theta^{[\pi_k(j)]})$.

Let $u^{[t]} := Z_{t+1} Z_{t+2} \ldots Z_{T-1} u$.
LIE based on the estimator $\Delta \theta_{-j}$ is then obtained as 
\begin{align*}
    \langle u, \Delta \theta_{-j} \rangle = \sum_{k=1}^K \langle u^{[\pi_k(j)]}, \frac{\eta_{\pi_k(j)}}{|S_{\pi_k(j)}|} g(z_j; \theta^{[\pi_k(j)]}) \rangle .
\end{align*}
It should be noted that $u^{[t]}$ can be computed recursively $u^{[t]} \leftarrow Z_{t+1} u^{[t+1]} = u^{[t+1]} - \eta_{t+1} H_{\theta^{[t+1]}} u^{[t+1]}$ by retracing SGD.
The proposed method is based on this recursive computation.

The proposed method consists of two phases, the training phase and the inference phase, as shown in Algorithms~\ref{alg:lie_train} and \ref{alg:lie_infer}.
In the training phase in Algorithm~\ref{alg:lie_train}, during running SGD, we store the tuple of the instance indices $S_t$, learning rate $\eta_t$, and parameter $\theta^{[t]}$.
In the inference phase in Algorithm~\ref{alg:lie_infer}, we retrace the stored information and compute $u^{[t]}$ in each step.

Note that, in Algorithm~\ref{alg:lie_infer}, we need to compute $H^{[t]} u^{[t]}$.
A naive implementation requires $O(p^2)$ memory to store the matrix $H^{[t]}$, which can be prohibitive for very large models.
We can avoid this difficulty by directly computing $H^{[t]} u^{[t]}$ without the explicit computation of $H^{[t]}$.
Because $H^{[t]} u^{[t]} = \frac{1}{|S_t|} \sum_{i \in S_t} \nabla_\theta \langle u^{[t]}, \nabla_\theta \ell(z_i; \theta^{[t]}) \rangle$, we only need to compute the derivative of $\langle u^{[t]}, \nabla_\theta \ell(z_i; \theta^{[t]}) \rangle$, which does not require the explicit computation of $H^{[t]}$.
For example, in Tensorflow, this can be implemented in a few lines.\footnote{\texttt{grads = [tf.gradients(loss[i], theta) for i in St]}; \texttt{Hu = tf.reduce\_mean(} \texttt{[tf.gradients(tf.tensordot(u, g, axes), theta) for g in grads], axis)}}
The time complexity for the inference phase is $O(TM\delta)$, where $M$ is the largest batch size in SGD and $\delta$ is the complexity for computing the parameter gradient.

%\begin{wrapfigure}{R}{.47\textwidth}
\begin{figure}[t]
%\vspace{-12pt}
\begin{minipage}[t]{.49\linewidth}
    \begin{algorithm}[H]
    \caption{LIE for SGD: Training Phase}
    \label{alg:lie_train}
    \begin{algorithmic}
    \STATE{Initialize the parameter $\theta^{[1]}$}
    \STATE{Initialize the sequence as null: $A\leftarrow \emptyset$}
    \FOR{$t = 1, 2, \ldots, T-1$}
    \STATE{$A[t] \leftarrow (S_t, \eta_t, \theta^{[t]})$ \; \small \textit{// store information}}
    \STATE{$\theta^{[t+1]} \leftarrow \theta^{[t]} - \frac{\eta_t}{|S_t|} \sum_{i \in S_t} g(z_i; \theta^{[t]})$}
    \ENDFOR
    \end{algorithmic}
    \end{algorithm}
\end{minipage}
\begin{minipage}[t]{.49\linewidth}
    \vspace{-9pt}
    \begin{algorithm}[H]
    \caption{LIE for SGD: Inference Phase}
    \label{alg:lie_infer}
    \begin{algorithmic}
    \REQUIRE{$u \in \mathbb{R}^p$}
    \STATE{Initialize the influence: $\hat{L}_{-j}^{[T]}(u) \leftarrow 0, \forall j$}
    \FOR{$t = T-1, T-2, \ldots, 1$}
    \STATE{$(S_t,  \eta_t, \theta^{[t]}) \leftarrow A[t]$ \; \textit{\small // load information}}
    \STATE{\textit{\small // update the linear influence of $z_j$}}
    \STATE{$\hat{L}_{-j}^{[T]}(u) \mathrel{+}= \langle u, \frac{\eta_t}{|S_t|} g(z_j; \theta^{[t]}) \rangle, \forall j \in S_t$}
    \STATE{$u \mathrel{-}= \eta_t H^{[t]} u$ \; \textit{\small // update $u$}}
    \ENDFOR
    \end{algorithmic}
    \end{algorithm}
\end{minipage}
%\vspace{-60pt}
%\end{wrapfigure}
\end{figure}

%%%%%%%%%%%%%%%%%%%%%%%%%%%%%%%%%%%%%%%%%%%%%%%%%%%%%%%%%%%%%%%%%%%%
\section{Related Studies}
\label{sec:related}

\paragraph{Influence Estimation}
Traditional studies on influence estimation considered the change in the solution $\hat{\theta}$ to the problem (\ref{eq:obj}) if an instance $z_j$ was absent.
For this purpose, they considered the counterfactual problem $\hat{\theta}_{-j} = \argmin_\theta \sum_{n=1; n \neq j}^N \ell(z; \theta)$.
The goal of the traditional influence estimation is to obtain an estimate of the difference $\hat{\theta}_{-j} - \hat{\theta}$ without retraining the models.
Pioneering studies by \cite{beckman1974distribution},\cite{cook1977detection}, and \cite{pregibon1981logistic} have shown that the influence $\hat{\theta}_{-j} - \hat{\theta}$ can be computed analytically for linear and generalized linear models.
\cite{koh2017understanding} considered a further generalizations of those previous studies.
They introduced the following approximation for strongly convex loss functions $\ell(z; \theta)$:
\begin{align}
    \textstyle \hat{\theta}_{-j} - \hat{\theta} \approx \frac{1}{N} \hat{H}^{-1} \nabla_\theta \ell(z_j; \hat{\theta}) ,
    \label{eq:icml}
\end{align}
where $\hat{H} = \frac{1}{N} \sum_{z \in D} \nabla^2 \ell(z; \hat{\theta})$ is the Hessian of the loss for the optimal model. 
We note that \cite{zhang2018training} and \cite{rajiv2019interpreting} further extended this approach.
\cite{zhang2018training} used this approach to fix the labels of the training instances.
\cite{rajiv2019interpreting} proposed to find the influential instances using the Bayesian quadrature, which includes (\ref{eq:icml}) as its special case.

Our study differs from these traditional approaches in two ways.
First, the proposed SGD-influence does not assume the optimality of the obtained models.
We instead consider the models obtained in each step of SGD, which are not necessarily optimal.
Second, the proposed method does not require the function loss $\ell(z; \theta)$ to be convex.
The proposed method is valid even for non-convex losses.

\paragraph{Outlier Detection}
A typical approach for data cleansing is outlier detection.
Outlier detection is used to remove abnormal instances from the training set before training the model to ensure that the model is not affected by the abnormal instances.
For tabular data, there are several popular methods such as One-class SVM~\citep{scholkopf2001estimating}, Local Outlier Factor~\citep{breunig2000lof}, and Isolation Forest~\citep{liu2008isolation}.
For complex data such as images, autoencoders can also be used~\citep{aggarwal2016outlier,zhou2017anomaly} along with generative adversarial networks~\citep{schlegl2017unsupervised}.
It should be noted that although these methods can find abnormal instances, they are not necessarily influential to the resulting models, as we will show in the experiments.

%%%%%%%%%%%%%%%%%%%%%%%%%%%%%%%%%%%%%%%%%%%%%%%%%%%%%%%%%%%%%%%%%%%%
\section{Experiments}
\label{sec:exp}

Here, we evaluate the two aspects of the proposed method: the performances of LIE and data cleansing.
We used Python 3 and PyTorch 1.0 for the experiments.\footnote{The codes are available at \url{https://github.com/sato9hara/sgd-influence}}
The experiments were conducted on 64bit Ubuntu 16.04 with six Intel Xeon E5-1650 3.6GHz CPU, 128GB RAM, and four GeForce GTX 1080ti.

%%%%%%%%%%%%%%%%%%%%%%%%%%%%%%%%%%%%%%%%%%%%%%%%%%%%%%%%%%%%%%%%%%%%
\subsection{Evaluation of LIE}
\label{sec:exp_lie}

We first evaluate the effectiveness of the proposed method in the estimation of linear influence.
For this purpose, we artificially created small datasets to ensure that the true linear influence is computable.
The detailed setup can be found in Appendix \ref{app:exp_lie}.

\paragraph{Setup}
We used three datasets: Adult~\citep{Dua2017}, 20Newsgroups\footnote{\url{http://qwone.com/~jason/20Newsgroups/}}, and MNIST~\citep{lecun1998gradient}.
These are common benchmarks in tabular data analysis, natural language processing, and image recognition, respectively.
We adopted these three datasets to demonstrate the validity of the proposed method across different data domains.
For 20Newsgroups and MNIST, we selected the two document categories \texttt{ibm.pc.hardware} and \texttt{mac.hardware} and images from one and seven, respectively, so that the problem to be binary classification. 

To observe the validity of the proposed method beyond convexity, we adopted two models, linear logistic regression and deep neural networks.
For deep neural networks, we used a network with two fully connected layers with eight units each and ReLU activation.
We used the sigmoid function at the output layer and adopted the cross entropy as the loss function.
It should be noted that the loss function for linear logistic regression is convex, while that for deep neural networks is non-convex.

In the experiments, we randomly subsampled 200 instances for the training set $D$ and validation set $D'$.
We then estimated the linear influence for the validation loss using Algorithm~\ref{alg:lie_infer}.
Here, we set the query vector $u$ as $u = \frac{1}{|D'|}\sum_{z' \in D'} \nabla_\theta \ell(z'; \theta^{[T]})$.
The estimation of linear influence thus amounts to estimating the change in the validation loss $\langle u, \theta_{-j}^{[T]} - \theta^{[T]} \rangle \approx \frac{1}{|D'|}\sum_{z' \in D'} \left( \ell(z'; \theta_{-j}^{[T]}) - \ell(z'; \theta^{[T]}) \right)$.

\paragraph{Evaluation}
We ran the counterfactual SGD for all $z_j \in D$ and computed the true linear influence.
For evaluation, we compared the estimated influences with this true influence using Kendall's tau and Jaccard index.
With Kendall's tau, a typical metric for ordinal associations, we measured the correlation between the estimated and true influences.
Kendall's tau takes the value between plus and minus one, where one indicates that the orders of the estimated and true influences are identical.
With Jaccard index, we measured the identification accuracy of the influential instances.
For data cleansing, the users are interested in instances with large positive or negative influences.
We selected ten instances with the largest positive and negative true influences and constructed a set of 20 important instances.
We compared this important instances with the estimated ones using Jaccard index, which varies between zero and one, where the value one indicates that the sets are identical.

\paragraph{Results}
We adopted the method proposed by \cite{koh2017understanding} in (\ref{eq:icml}) as the baseline, abbreviated as K\&L.
For deep neural networks, the Hessian matrix is not positive definite, which makes the estimator (\ref{eq:icml}) invalid.
To alleviate the effect of negative eigenvalues, we added a positive constant $1.0$ to the diagonal as suggested by \cite{koh2017understanding}.

Figure~\ref{fig:lie} shows a clear advantage of the proposed method.
The proposed method successfully estimated the true linear influences with high precision.
The estimated influences were concentrated on the diagonal lines, indicating that the estimated influences accurately approximated the true influences.
In contrast, the estimated influences obtained by K\&L were less accurate.
We observed that the estimator (\ref{eq:icml}) sometimes gets numerically unstable owing to the presence of small eigenvalues in the Hessian matrix.

For the quantitative comparison, we repeated the experiment by randomly changing the instance subsampling 100 times.
Table~\ref{tab:res} lists the average Kendall's tau and Jaccard index.
The results again show that the proposed method can accurately estimate the true linear influences.

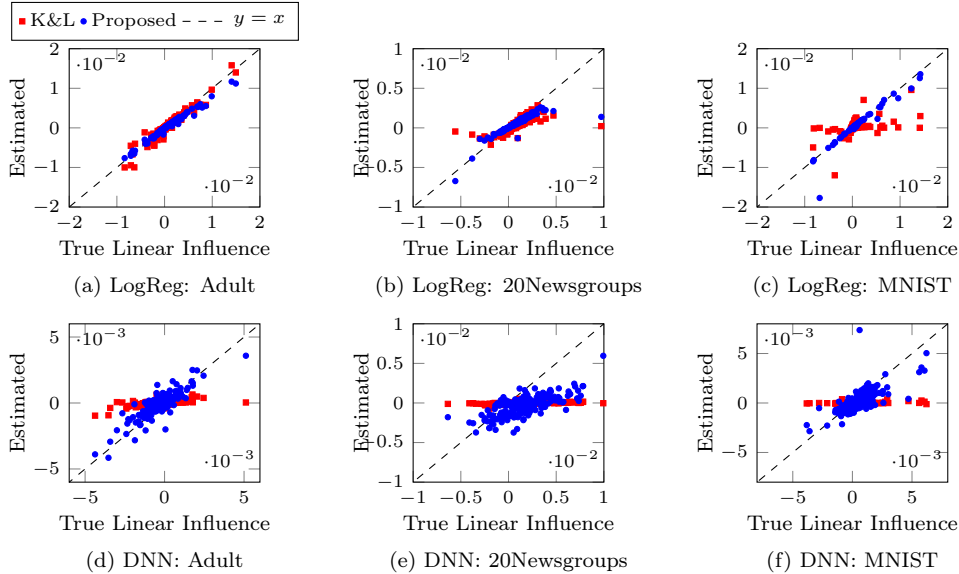
\begin{figure*}[t]
\centering
\begin{tikzpicture}
\begin{axis}[
name=axis1,
scale=0.37,
xmin=-0.02,
xmax=0.02,
ymin=-0.02,
ymax=0.02,
xlabel=True Linear Influence,
ylabel=Estimated,
x label style={at={(axis description cs:0.5,0.1)},anchor=north,font=\footnotesize},
y label style={at={(axis description cs:0.3,0.5)},anchor=south,font=\footnotesize},
scaled ticks=true,
every x tick scale label/.style={
    at={(xticklabel* cs:0.84,-0.55cm)},
    anchor=near xticklabel
},
every y tick scale label/.style={
    at={(yticklabel* cs:0.92,-0.9cm)},
    anchor=near yticklabel
},
tick label style={font=\scriptsize},
legend style={font=\fontsize{7}{5}\selectfont},
legend style={at={(0.45,1.05)},anchor=south,legend columns=-1}
]
\addplot[red, only marks, mark=square*, mark size=1pt] table[x=true, y=icml, col sep=comma]{adult_loss_diff.csv};
\addplot[blue, only marks, mark=*, mark size=1pt] table[x=true, y=proposed, col sep=comma]{adult_loss_diff.csv};
\addplot [black, dashed, domain=-0.02:0.02, samples=3] {x};
\legend{K\&L, Proposed, $y=x$}
\end{axis}
\begin{axis}[
name=axis2,
xshift = 0.3\textwidth,
scale=0.37,
xmin=-0.010,
xmax=0.010,
ymin=-0.010,
ymax=0.010,
xlabel=True Linear Influence,
ylabel=Estimated,
x label style={at={(axis description cs:0.5,0.1)},anchor=north,font=\footnotesize},
y label style={at={(axis description cs:0.3,0.5)},anchor=south,font=\footnotesize},
scaled ticks=true,
every x tick scale label/.style={
    at={(xticklabel* cs:0.84,-0.55cm)},
    anchor=near xticklabel
},
every y tick scale label/.style={
    at={(yticklabel* cs:0.92,-0.9cm)},
    anchor=near yticklabel
},
tick label style={font=\scriptsize},
]
\addplot[red, only marks, mark=square*, mark size=1pt] table[x=true, y=icml, col sep=comma]{20news_loss_diff.csv};
\addplot[blue, only marks, mark=*, mark size=1pt] table[x=true, y=proposed, col sep=comma]{20news_loss_diff.csv};
\addplot [black, dashed, domain=-0.02:0.02, samples=3] {x};
\end{axis}
\begin{axis}[
name=axis3,
xshift = 0.6\textwidth,
scale=0.37,
xmin=-0.02,
xmax=0.02,
ymin=-0.02,
ymax=0.02,
xlabel=True Linear Influence,
ylabel=Estimated,
x label style={at={(axis description cs:0.5,0.1)},anchor=north,font=\footnotesize},
y label style={at={(axis description cs:0.3,0.5)},anchor=south,font=\footnotesize},
scaled ticks=true,
every x tick scale label/.style={
    at={(xticklabel* cs:0.84,-0.55cm)},
    anchor=near xticklabel
},
every y tick scale label/.style={
    at={(yticklabel* cs:0.92,-0.9cm)},
    anchor=near yticklabel
},
tick label style={font=\scriptsize},
]
\addplot[red, only marks, mark=square*, mark size=1pt] table[x=true, y=icml, col sep=comma]{mnist_loss_diff.csv};
\addplot[blue, only marks, mark=*, mark size=1pt] table[x=true, y=proposed, col sep=comma]{mnist_loss_diff.csv};
\addplot [black, dashed, domain=-0.1:0.1, samples=3] {x};
\end{axis}
\node [below=0.8cm] at (axis1.south) {\footnotesize (a) LogReg: Adult};
\node [below=0.8cm] at (axis2.south) {\footnotesize (b) LogReg: 20Newsgroups};
\node [below=0.8cm] at (axis3.south) {\footnotesize (c) LogReg: MNIST};
\end{tikzpicture}
\centering
\begin{tikzpicture}
\begin{axis}[
name=axis1,
scale=0.37,
xmin=-0.006,
xmax=0.006,
ymin=-0.006,
ymax=0.006,
xlabel=True Linear Influence,
ylabel=Estimated,
x label style={at={(axis description cs:0.5,0.1)},anchor=north,font=\footnotesize},
y label style={at={(axis description cs:0.3,0.5)},anchor=south,font=\footnotesize},
scaled ticks=true,
every x tick scale label/.style={
    at={(xticklabel* cs:0.84,-0.55cm)},
    anchor=near xticklabel
},
every y tick scale label/.style={
    at={(yticklabel* cs:0.92,-0.9cm)},
    anchor=near yticklabel
},
tick label style={font=\scriptsize},
legend style={font=\fontsize{7}{5}\selectfont},
legend style={at={(0.45,1.05)},anchor=south,legend columns=-1}
]
\addplot[red, only marks, mark=square*, mark size=1pt] table[x=true, y=icml, col sep=comma]{adult_loss_diff_dnn.csv};
\addplot[blue, only marks, mark=*, mark size=1pt] table[x=true, y=proposed, col sep=comma]{adult_loss_diff_dnn.csv};
\addplot [black, dashed, domain=-0.02:0.03, samples=3] {x};
\end{axis}
\begin{axis}[
name=axis2,
xshift = 0.3\textwidth,
scale=0.37,
xmin=-0.010,
xmax=0.010,
ymin=-0.010,
ymax=0.010,
xlabel=True Linear Influence,
ylabel=Estimated,
x label style={at={(axis description cs:0.5,0.1)},anchor=north,font=\footnotesize},
y label style={at={(axis description cs:0.3,0.5)},anchor=south,font=\footnotesize},
scaled ticks=true,
every x tick scale label/.style={
    at={(xticklabel* cs:0.84,-0.55cm)},
    anchor=near xticklabel
},
every y tick scale label/.style={
    at={(yticklabel* cs:0.92,-0.9cm)},
    anchor=near yticklabel
},
tick label style={font=\scriptsize},
]
\addplot[red, only marks, mark=square*, mark size=1pt] table[x=true, y=icml, col sep=comma]{20news_loss_diff_dnn.csv};
\addplot[blue, only marks, mark=*, mark size=1pt] table[x=true, y=proposed, col sep=comma]{20news_loss_diff_dnn.csv};
\addplot [black, dashed, domain=-0.3:0.3, samples=3] {x};
\end{axis}
\begin{axis}[
name=axis3,
xshift = 0.6\textwidth,
scale=0.37,
xmin=-0.008,
xmax=0.008,
ymin=-0.008,
ymax=0.008,
xlabel=True Linear Influence,
ylabel=Estimated,
x label style={at={(axis description cs:0.5,0.1)},anchor=north,font=\footnotesize},
y label style={at={(axis description cs:0.3,0.5)},anchor=south,font=\footnotesize},
scaled ticks=true,
every x tick scale label/.style={
    at={(xticklabel* cs:0.84,-0.55cm)},
    anchor=near xticklabel
},
every y tick scale label/.style={
    at={(yticklabel* cs:0.92,-0.9cm)},
    anchor=near yticklabel
},
tick label style={font=\scriptsize},
]
\addplot[red, only marks, mark=square*, mark size=1pt] table[x=true, y=icml, col sep=comma]{mnist_loss_diff_dnn.csv};
\addplot[blue, only marks, mark=*, mark size=1pt] table[x=true, y=proposed, col sep=comma]{mnist_loss_diff_dnn.csv};
\addplot [black, dashed, domain=-0.015:0.015, samples=3] {x};
\end{axis}
\node [below=0.8cm] at (axis1.south) {\footnotesize (d) DNN: Adult};
\node [below=0.8cm] at (axis2.south) {\footnotesize (e) DNN: 20Newsgroups};
\node [below=0.8cm] at (axis3.south) {\footnotesize (f) DNN: MNIST};
\end{tikzpicture}
\caption{Estimated linear influences for linear logistic regression (LogReg) and deep neural networks (DNN) for all the 200 training instances. K\&L denotes the method of \protect\cite{koh2017understanding}.}
\label{fig:lie}
\end{figure*}

\begin{table}[t]
\small
\centering
\caption{Average Kendall's tau and Jaccard index ($\pm$ std.).}
\label{tab:res}
\begin{tabular}{ccccccccc}
\toprule
& \multicolumn{4}{c}{Kendall's tau} & \multicolumn{4}{c}{Jaccard index} \\
\cmidrule(r){2-9}
& \multicolumn{2}{c}{LogReg} & \multicolumn{2}{c}{DNN} & \multicolumn{2}{c}{LogReg} & \multicolumn{2}{c}{DNN} \\
& Proposed & K\&L & Proposed & K\&L & Proposed & K\&L & Proposed & K\&L \\
\midrule
Adult & .93 (.02) & .85 (.07) & .75 (.10) & .54 (.12) & .80 (.10) & .60 (.17) & .59 (.16) & .32 (.11) \\
20News & .94 (.05) & .82 (.15) & .45 (.12) & .37 (.12) & .79 (.15) & .52 (.19) & .25 (.08) & .11 (.07) \\
MNIST & .95 (.02)  & .70 (.15) & .45 (.12)  & .27 (.16) & .83 (.10)  & .41 (.16) & .37 (.15)  & .27 (.12) \\
\bottomrule
\end{tabular}
\end{table}

%%%%%%%%%%%%%%%%%%%%%%%%%%%%%%%%%%%%%%%%%%%%%%%%%%%%%%%%%%%%%%%%%%%%
\subsection{Evaluation on Data Cleansing}
\label{sec:exp_cleans}

We now show that the proposed method is effective for data cleansing.
Specifically, on MNIST~\citep{lecun1998gradient} and CIFAR10~\citep{krizhevsky2009learning}, we demonstrate that we can effectively improve the models by removing influential instances suggested by the proposed method.
The detailed setup and full results can be found in Appendix \ref{app:exp_cleans} and \ref{app:exp_cleans_res}.

\paragraph{Setup}
We used MNIST and CIFAR10.
From the original training set, we held out randomly selected 10,000 instances for the validation set and used the remaining instances as the training set.
As models, we used convolutional neural networks.
In SGD, we set the epoch $K=20$, batch size $|S_t| = 64$, and learning rate $\eta_t = 0.05$.

As baselines for data cleansing, in addition to K\&L, we adopted two outlier detection methods, Autoencoder~\citep{aggarwal2016outlier} and Isolation Forest~\citep{liu2008isolation}.
We also adopted random data removal as the baseline.
For the proposed method, we introduced an approximate version in this experiment.
In Algorithm~\ref{alg:lie_infer}, the proposed method retraces all steps of the SGD.
In the approximate version, we retrace only one epoch, which requires less computation than the original algorithm.
Moreover, it is also storage friendly because we need to store intermediate information only in the last epoch of SGD.

We proceeded the experiment as follows.
First, we trained the model with SGD using the training set.
We then computed the influence of each training instance using the proposed method as well as other baseline methods.
Here, we used the same query vector $u$ as in the previous experiment.
Finally, we removed the top-$m$ influential instances from the training set and retrained the model.
For model retraining, we ran normal SGD for 19 epochs and switched to counterfactual SGD in the last epoch.\footnote{We observed that this works well. For the results with full counterfactual SGD, see Apendix~\ref{app:exp_cleans_res}.}
If the misclassification rate of the retrained model decreases, we can conclude that the training set was effectively cleansed. 

\paragraph{Results}
We repeated the experiment by randomly changing the split between the training and validation set 30 times.
\figurename~\ref{fig:mis} shows the misclassification rates on the test set after data cleansing with each method.\footnote{See Appendix~\ref{app:exp_cleans_res} for the full results.}
It is evident from the figures that the misclassification rates decreased after data cleansing with the proposed method and its approximate version.
We compared the misclassification rates before and after the data cleansing using t-test with the significance level set to $0.05$.
We observed that none of the baseline methods except K\&L attained statistically significant improvements.
By contrast, the proposed method and its approximate version attained statistically significant improvements.
For both datasets, the proposed method and its approximate version were found to be statistically significant for the number of removed instances between 10 and 1000, and 10 and 100, respectively.
Moreover, both methods outperformed K\&L.
The results confirm that the proposed method can effectively suggest influential instances for data cleansing.
We also note that the proposed method and its approximate version performed comparably well.
This observation suggests that, in practice, we only need to retrace only one epoch for inferring the influential instances, which requires less computation and storing intermediate information only in the last epoch of SGD.

\figurename~\ref{fig:example} shows examples of found influential instances.
An interesting observation is that Autoencoder tended to find images with noisy or vivid backgrounds.
Visually, it seems reasonable to select them as outliers.
However, as we have seen in \figurename~\ref{fig:mis}, removing these outliers did not help to improve the models.
In contrast, the proposed method found images with confusing shapes or backgrounds.
Although they are not strongly visually appealing as the outliers, \figurename~\ref{fig:mis} confirms that these instances significantly affect the models.
These observations indicate that the proposed method could find the influential instances, which can be missed even by users with domain knowledge.

\begin{figure}[t]
\begin{tikzpicture}
\centering
\begin{axis}[
name=axis1,
width=0.45\textwidth,
height=120pt,
xmode=log,
xmin=1,
xmax=10000,
ymin=0.007,
ymax=0.011,
xlabel=\# of instances removed,
ylabel=Misclassification rate,
label style={font=\footnotesize},
scaled ticks=false,
tick label style={
	/pgf/number format/fixed,
    /pgf/number format/precision=4,
    font=\scriptsize
},
legend style={font=\fontsize{7}{5}\selectfont},
legend style={at={(1.1,1.08)},anchor=south,legend columns=-1}
]
\addplot[black, densely dotted, line width=0.5mm] table[x=k, y=baseline, col sep=comma]{./figs/mnist_miss.csv};
\addplot[magenta, mark=diamond*, line width=0.3mm] table[x=k, y=random, col sep=comma]{./figs/mnist_miss.csv};
\addplot[color={rgb:red,2;green,4;blue,2}, line width=0.3mm] table[x=k, y=ae, col sep=comma]{./figs/mnist_miss.csv};
\addplot[color={rgb:red,2;green,4;blue,2}, densely dotted, line width=0.3mm] table[x=k, y=iso, col sep=comma]{./figs/mnist_miss.csv};
\addplot[blue, mark=*, mark size=1.5pt, line width=0.3mm] table[x=k, y=sgd_all, col sep=comma]{./figs/mnist_miss.csv};
\addplot[blue, mark options=solid, densely dotted, mark=triangle*, mark size=2pt, line width=0.3mm] table[x=k, y=sgd_last, col sep=comma]{./figs/mnist_miss.csv};
\addplot[red, mark=square*, mark size=1.5pt, line width=0.3mm] table[x=k, y=icml, col sep=comma]{./figs/mnist_miss.csv};
\addplot[black, densely dotted, line width=0.5mm] table[x=k, y=baseline, col sep=comma]{./figs/mnist_miss.csv};
\legend{No Removal, Random, Autoencoder, Isolation Forest, Proposed, Proposed (Approx.), K\&L}
\end{axis}
\begin{axis}[
xshift=2.7in,
name=axis2,
width=0.45\textwidth,
height=120pt,
xmode=log,
xmin=1,
xmax=10000,
ymin=0.15,
ymax=0.19,
xlabel=\# of instances removed,
ylabel=Misclassification rate,
label style={font=\footnotesize},
scaled ticks=false,
tick label style={
	/pgf/number format/fixed,
    /pgf/number format/precision=4,
    font=\scriptsize
},
legend style={font=\fontsize{7}{5}\selectfont},
legend pos=north west
]
\addplot[black, densely dotted, line width=0.5mm] table[x=k, y=baseline, col sep=comma]{./figs/cifar10_miss.csv};
\addplot[magenta, mark=diamond*, line width=0.3mm] table[x=k, y=random, col sep=comma]{./figs/cifar10_miss.csv};
\addplot[color={rgb:red,2;green,4;blue,2}, line width=0.3mm] table[x=k, y=ae, col sep=comma]{./figs/cifar10_miss.csv};
\addplot[color={rgb:red,2;green,4;blue,2}, densely dotted, line width=0.3mm] table[x=k, y=iso, col sep=comma]{./figs/cifar10_miss.csv};
\addplot[blue, mark=*, mark size=1.5pt, line width=0.3mm] table[x=k, y=sgd_all, col sep=comma]{./figs/cifar10_miss.csv};
\addplot[blue, mark options=solid, densely dotted, mark=triangle*, mark size=2pt, line width=0.3mm] table[x=k, y=sgd_last, col sep=comma]{./figs/cifar10_miss.csv};
\addplot[red, mark=square*, mark size=1.5pt, line width=0.3mm] table[x=k, y=icml, col sep=comma]{./figs/cifar10_miss.csv};
\addplot[black, densely dotted, line width=0.5mm] table[x=k, y=baseline, col sep=comma]{./figs/cifar10_miss.csv};
\end{axis}
\node [below=1cm] at (axis1.south) {\footnotesize (a) MNIST};
\node [below=1cm] at (axis2.south) {\footnotesize (b) CIFAR10};
\end{tikzpicture}
\caption{Average misclassification rates on the test set after data cleansing. The errorbars are omitted for better visibility. See Appendix~\ref{app:exp_cleans_res} for the full results.}
\label{fig:mis}
%\end{figure}
%
%\begin{figure}[t]
    \centering
    \begin{tikzpicture}
    \node[anchor=center] (a) at (0, 0) {\includegraphics[width=0.23\textwidth]{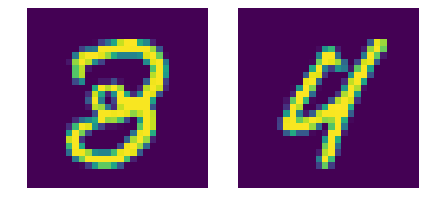}};
    \node[anchor=center, above left=-0.25cm and -1.5cm of a] {\footnotesize $y=3$};
    \node[anchor=center, above right=-0.25cm and -1.35cm of a] {\footnotesize $y=4$};
    \node[anchor=center, right=0.0cm of a] (b) {\includegraphics[width=0.23\textwidth]{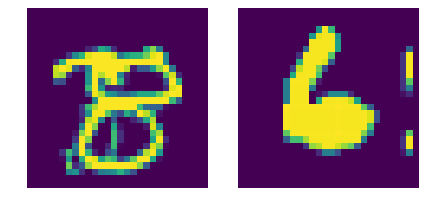}};
    \node[anchor=center, above left=-0.25cm and -1.5cm of b] {\footnotesize $y=8$};
    \node[anchor=center, above right=-0.25cm and -1.35cm of b] {\footnotesize $y=6$};
    \node[anchor=center, right=0.0cm of b] (c) {\includegraphics[width=0.23\textwidth]{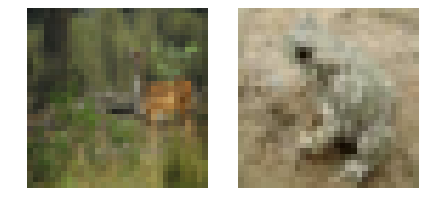}};
    \node[anchor=center, above left=-0.25cm and -1.7cm of c] {\footnotesize $y=$ deer};
    \node[anchor=center, above right=-0.25cm and -1.5cm of c] {\footnotesize $y=$ frog};
    \node[anchor=center, right=0.0cm of c] (d) {\includegraphics[width=0.23\textwidth]{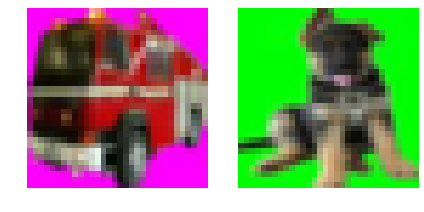}};
    \node[anchor=center, above left=-0.25cm and -1.75cm of d] {\footnotesize $y=$ truck};
    \node[anchor=center, above right=-0.25cm and -1.5cm of d] {\footnotesize $y=$ dog};
    \node[below=-0.2cm of a] {\footnotesize (a) Proposed (Approx.)};
    \node[below=-0.2cm of b] {\footnotesize (b) Autoencoder};
    \node[below=-0.2cm of c] {\footnotesize (c) Proposed (Approx.)};
    \node[below=-0.2cm of d] {\footnotesize (d) Autoencoder};
    \end{tikzpicture}
    \caption{Examples of found influential instances and their labels in (a)(b) MNIST and (c)(d) CIFAR10.}
    \label{fig:example}
\end{figure}

%%%%%%%%%%%%%%%%%%%%%%%%%%%%%%%%%%%%%%%%%%%%%%%%%%%%%%%%%%%%%%%%%%%%
\section{Conclusion}

We considered supporting non-experts to build accurate machine learning models through data cleansing by suggesting influential instances.
Specifically, we aimed at establishing an algorithm that can infer the influential instances even for non-convex loss functions such as deep neural networks.
Our idea is to use the fact that modern machine learning models are trained using SGD.
We introduced a refined notion of influence for the models trained with SGD, which was named SGD-influence.
We then proposed an algorithm that can accurately approximate the SGD-influence without running extra SGD.
We also proved that the proposed method can provide valid estimates even for non-convex loss functions.
The experimental results have shown that the proposed method can accurately infer influential instances.
Moreover, on MNIST and CIFAR10, we demonstrated that the models can be effectively improved by removing the influential instances suggested by the proposed method.

\bibliographystyle{named}
\bibliography{main}

%%%%%%%%%%%%%%%%%%%%%%%%%%%%%%%%%%%%%%%%%%%%%%%%%%%%%%%%%%%%%%%%%%%%
%%%%%%%%%%%%%%%%%%%%%%%%%%%%%%%%%%%%%%%%%%%%%%%%%%%%%%%%%%%%%%%%%%%%
\clearpage
\appendix

%%%%%%%%%%%%%%%%%%%%%%%%%%%%%%%%%%%%%%%%%%%%%%%%%%%%%%%%%%%%%%%%%%%%
\section{Relation to \cite{koh2017understanding}}
\label{app:rel}

\subsection{Brief Review}
As we mentioned in Section~\ref{sec:related}, \cite{koh2017understanding} proposed to estimate the influence by (\ref{eq:icml}), which is
\begin{align*}
    \hat{\theta}_{-j} - \hat{\theta} \approx \frac{1}{N} \hat{H}^{-1} \nabla_\theta \ell(z_j; \hat{\theta}) ,
\end{align*}
where $\hat{H} = \frac{1}{N} \sum_{z \in D} \nabla^2 \ell(z; \hat{\theta})$ is the Hessian of the problem (\ref{eq:obj}) for the optimal model $\hat{\theta}$.

Note that, $\hat{H}^{-1} \nabla_\theta \ell(z_j; \hat{\theta})$ is equivalent to the solution to the following optimization problem:
\begin{align}
    \argmin_{\beta \in \mathbb{R}^p} \frac{1}{2} \langle \beta, \hat{H} \beta \rangle - \langle \nabla_\theta \ell(z_j; \hat{\theta}), \beta \rangle .
    \label{eq:kl}
\end{align}
\cite{koh2017understanding} proposed computing $\hat{H}^{-1} \nabla_\theta \ell(z_j; \hat{\theta})$ by solving this optimization problem using conjugate gradient descent or its improved version.
In the optimization, they also proposed to use the mini-batch approximation of the Hessian matrix: they proposed to use $\hat{H}_S = \frac{1}{|S|} \sum_{z \in S} \nabla^2 \ell(z; \hat{\theta})$ on the mini-batch $S \subset D$ instead of the Hessian matrix $\hat{H}$ computed on the all training instances $D$.

\subsection{Relation to the Proposed Method}
Here, we show the relationship between the proposed method and the method of \cite{koh2017understanding}.
Suppose that we solve the problem (\ref{eq:kl}) using SGD.
In the $t$-th step of SGD, we update $\beta$ by
\begin{align*}
    \beta^{[t+1]} = \beta^{[t]} - \gamma_t (\hat{H}_{S_t} \beta^{[t]} - \nabla_\theta \ell(z_j; \hat{\theta})) = (I - \gamma_t \hat{H}_{S_t}) \beta^{[t]} + \gamma_t \nabla_\theta \ell(z_j; \hat{\theta}) ,
\end{align*}
where $S_t$ is the mini-batch and $\gamma_t > 0$ is a learning rate.
Suppose that we initialized $\beta^{[1]} = \nabla_\theta \ell(z_j; \hat{\theta})$ and $\gamma := \max_t \gamma_t$.
Then, the SGD for the problem (\ref{eq:kl}) can be expressed as
\begin{align*}
    \beta^{[2]} & = (I - \gamma_1 \hat{H}_{S_1}) \nabla_\theta \ell(z_j; \hat{\theta}) + \gamma_1 \nabla_\theta \ell(z_j; \hat{\theta}), \\
    \beta^{[3]} & = (I - \gamma_2 \hat{H}_{S_2}) \beta^{[2]} + \gamma_2 \nabla_\theta \ell(z_j; \hat{\theta}) = (I - \gamma_2 \hat{H}_{S_2}) (I - \gamma_1 \hat{H}_{S_1}) \nabla_\theta \ell(z_j; \hat{\theta}) + O(\gamma) , \\
    \vdots & \\
    \beta^{[T]} &= (I - \gamma_{T-1} \hat{H}_{S_{T-1}})(I - \gamma_{T-2} \hat{H}_{S_{T-2}}) \ldots (I - \gamma_1 \hat{H}_{S_1}) \nabla_\theta \ell(z_j; \hat{\theta}) + O(\gamma) .
\end{align*}
Here, let $\hat{Z}_t := I - \gamma_{t} \hat{H}_{S_{t}}$, and we obtain
\begin{align*}
    \hat{\theta}_{-j} - \hat{\theta} \approx \frac{1}{N} \hat{H}^{-1} \nabla_\theta \ell(z_j; \hat{\theta}) \approx \beta^{[T]} = \frac{1}{N}  \hat{Z}_{T-1} \hat{Z}_{T-2} \ldots \hat{Z}_{1} \nabla_\theta \ell(z_j; \hat{\theta}) + O\left(\frac{\gamma}{N}\right) .
\end{align*}
When $\gamma$ is small and the last term is ignorable, this result resembles to the proposed estimator $\Delta \theta_{-j}$ in Section~\ref{sec:est}.
Instead of $\hat{Z}_t := I - \gamma_{t} \hat{H}_{S_{t}}$ computed at the optimal model $\hat{\theta}$, the proposed estimator uses $Z_t = I - \eta_t H^{[t]}$ based on the model $\theta^{[t]}$ in the $t$-th SGD step in the training.

%%%%%%%%%%%%%%%%%%%%%%%%%%%%%%%%%%%%%%%%%%%%%%%%%%%%%%%%%%%%%%%%%%%%
\section{Proof of Theorems}

Before proving Theorems~\ref{th:convex} and \ref{th:nonconvex}, we first prove the next lemma.

\begin{lemma}
\label{lem:hess}
Assume that $\ell(z; \theta)$ is twice differentiable with respect to the parameter $\theta$, and assume that there exist $\lambda, \Lambda > 0$ such that $\lambda I \prec \nabla^2_\theta \ell(z; \theta) \prec \Lambda I$ for all $z, \theta$.
If $\eta_s \leq 1/\Lambda$, then we get
\begin{align}
& h_j(\Lambda) \leq \|\Delta \theta_{-j}\| \leq h_j(\lambda), \label{eq:estimator_bound} \\
& h_j(\Lambda) \leq \|\theta_{-j}^{[T]} - \theta^{[T]}\| \leq h_j(\lambda) , \label{eq:provable_bound}
\end{align}
where $h_j(a) := \frac{\eta_{\pi(j)}}{|S_{\pi(j)}|}\prod_{s=\pi(j)+1}^{T-1}(1-\eta_{s} a) \| g(z_j; \theta^{[\pi(j)]}) \|$.
\end{lemma}

\begin{proof}
Since $(1-\eta_{s}\Lambda)I \prec Z_{s} \prec (1-\eta_{s}\lambda)I$ we immediately obtain (\ref{eq:estimator_bound}) 
from the definition (\ref{eq:diff_approx}) of $\Delta \theta_{-j}$.

We next show the inequality (\ref{eq:provable_bound}). 
There exists $r \in [0,1]$ such that for $\theta^{[s]}_* := r \theta^{[s]}_{-j} + (1-r)\theta^{[s]}$, 
\begin{align*}
    \frac{1}{|S_{s}|}\sum_{i \in S_{s}} \left(\nabla_\theta \ell(z_i; \theta_{-j}^{[s]}) - \nabla_\theta \ell(z_i; \theta^{[s]})\right) 
    = H^{[s]}_* (\theta_{-j}^{[s]} - \theta^{[s]}) ,
\end{align*}
where $H^{[s]}_* := \frac{1}{|S_{s}|} \sum_{i \in S_{s}} \nabla_\theta^2 \ell(z_i; \theta^{[s]}_*)$.
Therefore, by setting $Z_{s}^* := (I-\eta_s H^{[s]}_*)$, we can show the inequality (\ref{eq:provable_bound}) 
in a similar way to the proof of $(\ref{eq:estimator_bound})$.
\end{proof}

\subsection{Proof of Theorem~\ref{th:convex}}
\begin{proof}
From Lemma~\ref{lem:hess}, 
\begin{align*}
    \|(\theta_{-j}^{[T]} - \theta^{[T]}) - \Delta \theta_{-j}\|^2 &= \|\theta_{-j}^{[T]} - \theta^{[T]}\|^2 + \|\Delta \theta_{-j}\|^2 - 2 \langle \theta_{-j}^{[T]} - \theta^{[T]},  \Delta \theta_{-j} \rangle \\
    & \le h_j(\lambda)^2 + h_j(\lambda)^2 + 2 h_j(\Lambda)^2 = 2 (h_j(\lambda)^2 + h_j(\Lambda)^2) .
\end{align*}
\end{proof}

\subsection{Proof of Theorem~\ref{th:nonconvex}}
\begin{proof}
\[ \theta^{[s+1]}_{-j} - \theta^{[s+1]} = Z_{s} (\theta_{-j}^{[s]} - \theta^{[s]}) + \eta (H^{[s]}-H^{[s]}_*)(\theta_{-j}^{[s]} - \theta^{[s]}), \]
where $H^{[s]}_*$ is the same as that in the proof of Lemma \ref{lem:hess}.
We set $D_s := \eta (H^{[s]}-H^{[s]}_*)(\theta_{-j}^{[s]} - \theta^{[s]})$.
Applying this equalities recursively over $s \in \{\pi(j),\ldots, T-1\}$, we get
\[ \theta^{[T]}_{-j} - \theta^{[T]} 
= \Delta \theta_{-j} + \sum_{s=\pi(j)}^{T-1} \prod_{k=s+1}^{T-1}Z_k D_s. \]
Hence, a remaining problem is to bound the norm of the second term in the right hand side of this equality, which corresponds to a gap we want to evaluate.
Since $\|Z_k\| \leq 1 + \eta \Lambda$, $\|\theta_{-j}^{[s]} - \theta^{[s]}\| \leq 2\eta G T$
and $\|H^{[s]}-H^{[s]}_* \| \leq L \| \theta_{-j}^{[s]} - \theta^{[s]}\|$,
\begin{align*}
\left\| \sum_{s=\pi(j)}^{T-1} \prod_{k=s+1}^{T-1}Z_k D_s \right \| 
&\leq \sum_{s=1}^{T-1} \prod_{k=s+1}^{T-1} \| Z_k \| \|D_s\| \leq \sum_{s=1}^{T-1} \left(1 + \eta \Lambda \right)^{T-s-1} \eta L\| \theta_{-j}^{[s]} - \theta^{[s]}\|^{2} \\
&= 4\frac{\left(1 + \eta \Lambda \right)^{T-1} - 1}{  \left(1 + \eta \Lambda \right) - 1 } \eta^3 T^2 G^2 L \le 4\frac{\left(1 + O(\gamma \Lambda / \sqrt{T}) \right)^T}{\Lambda} \gamma^2 T G^2 L.
\end{align*}
\end{proof}

%%%%%%%%%%%%%%%%%%%%%%%%%%%%%%%%%%%%%%%%%%%%%%%%%%%%%%%%%%%%%%%%%%%%
\section{Details and Results of Experiments}

%%%%%%%%%%%%%%%%%%%%%%%%%%%%%%%%%%%%%%%%%%%%%%%%%%%%%%%%%%%%%%%%%%%%
\subsection{Setups in Section \ref{sec:exp_lie}}
\label{app:exp_lie}

\paragraph{Datasets}
We used three datasets: Adult~\citep{Dua2017}, 20Newsgroups\footnote{\url{http://qwone.com/~jason/20Newsgroups/}}, and MNIST~\citep{lecun1998gradient}.
These are common benchmarks in tabular data analysis, natural language processing, and image recognition, respectively.
We adopted these three datasets to demonstrate the validity of the proposed algorithm across different data domains.

We prepossessed each dataset as follows.
In Adult, we transformed categorical features to numerical attributes~\footnote{We used the implementation available at \url{https://www.kaggle.com/kost13/us-income-logistic-regression/notebook}}.
In 20Newsgroups, we selected the two document categories \texttt{ibm.pc.hardware} and \texttt{mac.hardware}.
As a preprocessing, we transformed the documents into numerical vectors using tf-idf, while removing frequent and scarce words.
In MNIST, we selected the images from the two categories one and seven, so that the problem to be binary classification.

\begin{wraptable}{r}{0.48\textwidth}
    %\vspace{-16pt}
    \centering
    \small
    \caption{Parameters used in SGD: $K$ denotes the number of epochs. $|S_t|$ denotes the batch size.}
    \label{tab:params}
    \begin{tabular}{ccccccc}
    \toprule
    & \multicolumn{3}{c}{LogReg} & \multicolumn{3}{c}{DNN} \\
    & $K$ & $|S_t|$ & $\eta_t$ & $K$ & $|S_t|$ & $\eta_t$ \\
    \midrule
    Adult & 20 & 5 & $\frac{0.1}{\sqrt{t}}$ & 10 & 20 & 0.1 \\
    20News & 10 & 5 & $\frac{0.01}{\sqrt{t}}$ & 10 & 20 & 0.1 \\
    MNIST & 5 & 5 & $\frac{0.1}{\sqrt{t}}$ & 10 & 20 & 0.1 \\
    \bottomrule
    \end{tabular}
    %\vspace{-12pt}
\end{wraptable}

\paragraph{Models}
To see the validity of the proposed method beyond convexity, we adopted two models, which are linear logistic regression and deep neural networks.
For linear logistic regression, we adopted the $\ell_2$-regularized loss $\ell(z;\theta) = \log(\exp(- y \langle \theta, x \rangle) + 1) + \frac{\alpha}{2} \|\theta\|^2$ where $y \in \{-1, 1\}$.
In the experiments, we determined the regularization parameter $\alpha$ using cross validation.
For deep neural networks, we used a network with two fully connected layers each of which has eight units with ReLU as an activation function.
We used the sigmoid function at the output layer, and adopted the cross entropy as the loss function.
To run SGD, we used the parameters shown in \tablename~\ref{tab:params}.
We note that the loss function for the linear logistic regression is convex, while that for the deep neural networks is non-convex.

\paragraph{Target Linear Influence}
In the experiments, we randomly subsampled 200 instances for the training set $D$ and the validation set $D'$.
We then estimated the linear influence for the validation loss using Algorithm~\ref{alg:lie_infer}.
Here, we set the query vector $u$ as $u = \frac{1}{|D'|}\sum_{z' \in D'} \nabla_\theta \ell(z'; \theta^{[T]})$.
Estimation of the linear influence thus amounts to estimating the change in the validation loss
\begin{align*}
    \langle u, \theta_{-j}^{[T]} - \theta^{[T]} \rangle \approx \frac{1}{|D'|}\sum_{z' \in D'} \left( \ell(z'; \theta_{-j}^{[T]}) - \ell(z'; \theta^{[T]}) \right).
\end{align*}
We note that the instances with large negative linear influences are deemed to be negatively affecting the resulting models.
Removing such instances can improve the validation loss, and thus the users can prioritize the inspection of such instances.

\paragraph{Baseline Method}
We adopted the method of \cite{koh2017understanding} as the baseline, which we abbreviated as K\&L.
In K\&L, we estimate the influence by (\ref{eq:icml}), which is
\begin{align*}
    \hat{\theta}_{-j} - \hat{\theta} \approx \frac{1}{N} \hat{H}^{-1} \nabla_\theta \ell(z_j; \hat{\theta}) ,
\end{align*}
where $\hat{H} = \frac{1}{N} \sum_{z \in D} \nabla^2 \ell(z; \hat{\theta})$ is the Hessian of the problem (\ref{eq:obj}) for the optimal model $\hat{\theta}$.
For a query vector $u \in \mathbb{R}^p$, the linear influence can be estimated as 
\begin{align*}
    \langle \hat{\theta}_{-j} - \hat{\theta}, u \rangle \approx \frac{1}{N} \langle \hat{H}^{-1} \nabla_\theta \ell(z_j; \hat{\theta}) , u \rangle = \frac{1}{N} \langle \nabla_\theta \ell(z_j; \hat{\theta}) , \hat{H}^{-1} u \rangle .
\end{align*}
Here, the last equality follows from the symmetricity of the Hessian matrix.
Thus, for estimating the linear influence for all the training instances, we first compute $\hat{H}^{-1} u$, and then take an inner product with $\nabla_\theta \ell(z_j; \hat{\theta})$ for each training instance $z_j \in D$.

Note that, $\hat{H}^{-1}u$ is equivalent to the solution to the following optimization problem:
\begin{align}
    \argmin_{\beta \in \mathbb{R}^p} \frac{1}{2} \langle \beta, \hat{H} \beta \rangle - \langle u, \beta \rangle .
    \label{eq:kl_u}
\end{align}
\cite{koh2017understanding} proposed computing $\hat{H}^{-1} u$ by solving this optimization problem using conjugate gradient descent or its improved version.
In the optimization, they also proposed to use the mini-batch approximation of the Hessian matrix: they proposed to use $\hat{H}_S = \frac{1}{|S|} \sum_{z \in S} \nabla^2 \ell(z; \hat{\theta})$ on the mini-batch $S \subset D$ instead of the Hessian matrix $\hat{H}$ computed on the all training instances $D$.
In the experiment, we ran momentum-SGD for 200 epochs, where we set the learning rate to be $0.01$, the size of momentum to be $0.9$, and the batch size to be $200$.

\paragraph{Evaluation Metrics}
In the experiments, we ran the counterfactual SGD for all $z_j \in D$, and computed the true linear influence.
We then used this ground truth to evaluate the goodness of the estimated linear influences.
For evaluation, we adopted the following two metrics.
The first metric is Kendall's tau.
Kendall's tau is a typical metric for measuring ordinal associations between two observations.
Kendall's tau takes the value between plus and minus one, where the value one indicates that the orders of the two observations are identical.

The second metric is Jaccard index.
For data cleansing, the users are interested in instances with large positive or negative influences.
We measured how accurately those important instances could be identified using the estimated influences.
To this end, we selected 10 instances with largest positive and negative true influences, and constructed a set of 20 important instances.
We compared this true important instances with the estimated important instances using Jaccard index.
Jaccard index measures the similarity of the two sets.
Jaccard index takes the value between zero and one, where the value one indicates that the sets are identical.

%%%%%%%%%%%%%%%%%%%%%%%%%%%%%%%%%%%%%%%%%%%%%%%%%%%%%%%%%%%%%%%%%%%%
\subsection{Setups in Section \ref{sec:exp_cleans}}
\label{app:exp_cleans}

\paragraph{Datasets}
We used MNIST~\citep{lecun1998gradient} and CIFAR10~\citep{krizhevsky2009learning}.
The MNIST dataset contains 60,000 training instances, while the CIFAR10 dataset contains 50,000 training instances.
Both datasets also contain 10,000 test instanes.
From the original training instances, we held out randomly selected 10,000 instances for the validation set, and used the remaining instances as the training set.
Thus, in the experiment, we used 50,000 instances in MNIST and 40,000 instances in CIFAR10 for training, and the held out 10,000 instances for validation.

\paragraph{Models}
We used convolutional neural networks (CNNs) in the experiment.
The network structures can be found in \figurename~\ref{fig:cnn}.
In SGD, we set the epochs $K=20$, batch size $|S_t| = 64$, and learning rate $\eta_t = 0.05$.
In the training, we used a simple data augmentation.
For MNIST, we applied horizontal and vertical shifts in $\pm 2$ pixels.
For CIFAR10, we applied horizontal and vertical shifts in $\pm 4$ pixels and horizontal flipping.

\begin{figure}[t]
    \centering
    \begin{tikzpicture}
        \node (input) at (0,0) [draw,thick,minimum width=2cm,minimum height=0.5cm] {input $x \in \mathbb{R}^{28 \times 28 \times 1}$};
        \node (conv1) [draw,thick,minimum width=2cm,minimum height=0.5cm, above=0.3cm of input] {Conv2D: size=$5 \times 5 \times 1$, \# of channels = $20$};
        \node (relu1) [draw,thick,minimum width=2cm,minimum height=0.5cm, above=0.3cm of conv1] {ReLU};
        \node (pool1) [draw,thick,minimum width=2cm,minimum height=0.5cm, above=0.3cm of relu1] {MaxPool2D: size = $2 \times 2$};
        \node (conv2) [draw,thick,minimum width=2cm,minimum height=0.5cm, above=0.3cm of pool1] {Conv2D: size=$5 \times 5 \times 1$, \# of channels = $20$};
        \node (relu2) [draw,thick,minimum width=2cm,minimum height=0.5cm, above=0.3cm of conv2] {ReLU};
        \node (pool2) [draw,thick,minimum width=2cm,minimum height=0.5cm, above=0.3cm of relu2] {MaxPool2D: size = $2 \times 2$};
        \node (flatten) [draw,thick,minimum width=2cm,minimum height=0.5cm, above=0.3cm of pool2] {Flatten};
        \node (fc) [draw,thick,minimum width=2cm,minimum height=0.5cm, above=0.3cm of flatten] {Fully Connected: size = $320 \times 10$};
        \node (softmax) [draw,thick,minimum width=2cm,minimum height=0.5cm, above=0.3cm of fc] {Softmax};
        \node (output) [draw,thick,minimum width=2cm,minimum height=0.5cm, above=0.3cm of softmax] {output $y \in \mathbb{R}^{10}$};
        \draw [->] (input.north) -- (conv1.south);
        \draw [->] (conv1.north) -- (relu1.south);
        \draw [->] (relu1.north) -- (pool1.south);
        \draw [->] (pool1.north) -- (conv2.south);
        \draw [->] (conv2.north) -- (relu2.south);
        \draw [->] (relu2.north) -- (pool2.south);
        \draw [->] (pool2.north) -- (flatten.south);
        \draw [->] (flatten.north) -- (fc.south);
        \draw [->] (fc.north) -- (softmax.south);
        \draw [->] (softmax.north) -- (output.south);
        \node [below=0.5cm of input.south] {(a) CNN for MNIST};
    %\end{tikzpicture}
    %\begin{tikzpicture}
    
        % first block
        \node (input) at (8,0) [draw,thick,minimum width=2cm,minimum height=0.5cm] {input $x \in \mathbb{R}^{32 \times 32 \times 3}$};
        \node (conv11) [draw,thick,minimum width=2cm,minimum height=0.5cm, above=0.3cm of input] {Conv2D: size=$3 \times 3 \times 3$, \# of channels = $32$};
        \node (bn1) [draw,thick,minimum width=2cm,minimum height=0.5cm, above=0.3cm of conv11] {BatchNorm2D};
        \node (relu11) [draw,thick,minimum width=2cm,minimum height=0.5cm, above=0.3cm of bn1] {ReLU};
        \node (conv12) [draw,thick,minimum width=2cm,minimum height=0.5cm, above=0.3cm of relu11] {Conv2D: size=$3 \times 3 \times 32$, \# of channels = $32$};
        \node (relu12) [draw,thick,minimum width=2cm,minimum height=0.5cm, above=0.3cm of conv12] {ReLU};
        \node (pool1) [draw,thick,minimum width=2cm,minimum height=0.5cm, above=0.3cm of relu12] {MaxPool2D: size = $2 \times 2$};
        
        % second block
        \node (conv21) [draw,thick,minimum width=2cm,minimum height=0.5cm, above=0.3cm of pool1] {Conv2D: size=$3 \times 3 \times 32$, \# of channels = $64$};
        \node (bn2) [draw,thick,minimum width=2cm,minimum height=0.5cm, above=0.3cm of conv21] {BatchNorm2D};
        \node (relu21) [draw,thick,minimum width=2cm,minimum height=0.5cm, above=0.3cm of bn2] {ReLU};
        \node (conv22) [draw,thick,minimum width=2cm,minimum height=0.5cm, above=0.3cm of relu21] {Conv2D: size=$3 \times 3 \times 64$, \# of channels = $64$};
        \node (relu22) [draw,thick,minimum width=2cm,minimum height=0.5cm, above=0.3cm of conv22] {ReLU};
        \node (pool2) [draw,thick,minimum width=2cm,minimum height=0.5cm, above=0.3cm of relu22] {MaxPool2D: size = $2 \times 2$};
        
        % third block
        \node (conv31) [draw,thick,minimum width=2cm,minimum height=0.5cm, above=0.3cm of pool2] {Conv2D: size=$3 \times 3 \times 64$, \# of channels = $128$};
        \node (bn3) [draw,thick,minimum width=2cm,minimum height=0.5cm, above=0.3cm of conv31] {BatchNorm2D};
        \node (relu31) [draw,thick,minimum width=2cm,minimum height=0.5cm, above=0.3cm of bn3] {ReLU};
        \node (conv32) [draw,thick,minimum width=2cm,minimum height=0.5cm, above=0.3cm of relu31] {Conv2D: size=$3 \times 3 \times 128$, \# of channels = $128$};
        \node (relu32) [draw,thick,minimum width=2cm,minimum height=0.5cm, above=0.3cm of conv32] {ReLU};
        \node (pool3) [draw,thick,minimum width=2cm,minimum height=0.5cm, above=0.3cm of relu32] {MaxPool2D: size = $2 \times 2$};
        
        \node (flatten) [draw,thick,minimum width=2cm,minimum height=0.5cm, above=0.3cm of pool3] {Flatten};
        \node (fc) [draw,thick,minimum width=2cm,minimum height=0.5cm, above=0.3cm of flatten] {Fully Connected: size = $2048 \times 10$};
        \node (softmax) [draw,thick,minimum width=2cm,minimum height=0.5cm, above=0.3cm of fc] {Softmax};
        \node (output) [draw,thick,minimum width=2cm,minimum height=0.5cm, above=0.3cm of softmax] {output $y \in \mathbb{R}^{10}$};
        \draw [->] (input.north) -- (conv11.south);
        \draw [->] (conv11.north) -- (bn1.south);
        \draw [->] (bn1.north) -- (relu11.south);
        \draw [->] (relu11.north) -- (conv12.south);
        \draw [->] (conv12.north) -- (relu12.south);
        \draw [->] (relu12.north) -- (pool1.south);
        \draw [->] (pool1.north) -- (conv21.south);
        \draw [->] (conv21.north) -- (bn2.south);
        \draw [->] (bn2.north) -- (relu21.south);
        \draw [->] (relu21.north) -- (conv22.south);
        \draw [->] (conv22.north) -- (relu22.south);
        \draw [->] (relu22.north) -- (pool2.south);
        \draw [->] (pool2.north) -- (conv31.south);
        \draw [->] (conv31.north) -- (bn3.south);
        \draw [->] (bn3.north) -- (relu31.south);
        \draw [->] (relu31.north) -- (conv32.south);
        \draw [->] (conv32.north) -- (relu32.south);
        \draw [->] (relu32.north) -- (pool3.south);
        \draw [->] (pool3.north) -- (flatten.south);
        \draw [->] (flatten.north) -- (fc.south);
        \draw [->] (fc.north) -- (softmax.south);
        \draw [->] (softmax.north) -- (output.south);
        \node [below=0.5cm of input.south] {(b) CNN for CIFAR10};
    \end{tikzpicture}
    \caption{Structures of convolutional neural networks (CNNs)}
    \label{fig:cnn}
\end{figure}
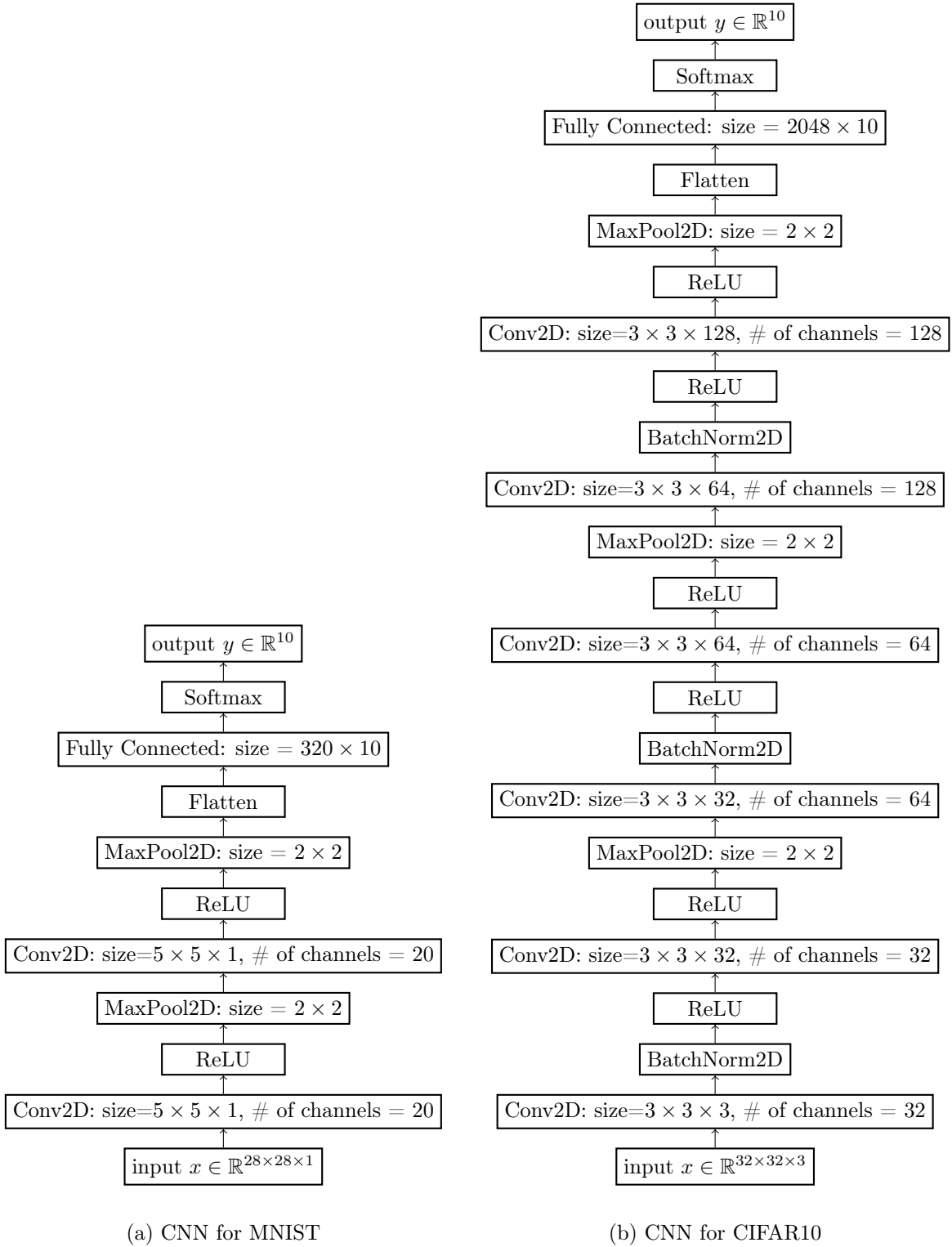

\paragraph{Target Linear Influence}
We set the query vector $u$ as $u = \frac{1}{|D'|}\sum_{z' \in D'} \nabla_\theta \ell(z'; \theta^{[T]})$.
Estimation of the linear influence thus amounts to estimating the change in the validation loss
\begin{align*}
    \langle u, \theta_{-j}^{[T]} - \theta^{[T]} \rangle \approx \frac{1}{|D'|}\sum_{z' \in D'} \left( \ell(z'; \theta_{-j}^{[T]}) - \ell(z'; \theta^{[T]}) \right).
\end{align*}
We note that the instances with large negative linear influences are deemed to be negatively affecting the resulting models.
Removing such instances can improve the validation loss, and thus the users can prioritize the inspection of such instances.

\paragraph{Baseline Methods}
For K\&L~\citep{koh2017understanding}, to solve the problem (\ref{eq:kl_u}), we ran momentum-SGD for two epochs, where we set the learning rate to be $0.005$, the size of momentum to be $0.9$, and the batch size to be $1000$.
As baselines for data cleansing, in addition to K\&L~\citep{koh2017understanding}, we also adopted two outlier detection methods, Autoencoder~\citep{aggarwal2016outlier} and Isolation Forest~\citep{liu2008isolation}.
In outlier detection, we treated the validation set as a healthy dataset.
We then computed outlierness of each training instance using outlier detection methods, as follows.
\begin{itemize}
    \item \textbf{Autoencoder}: We trained an autoencoder using the validation set. See \figurename~\ref{fig:ae} for the structures of autoencoders used. We adopted the squared loss as the training objective function. For training, we used Adam with the learning rate set to $0.001$ and the batch size set to $128$. We used the same data augmentation as the training of CNNs. After the autoencoder is trained, we fed each training input $x$ into the autoencoder and obtained an reconstructed input $\hat{x}$. We measured the outlierness of the input $x$ by $a = \|x - \hat{x}\|^2$.
    \item \textbf{Isolation Forest}: We first fed each validation input $x'$ into the trained CNN, and obtained its latent representation $r'$ from the flatten layer in \figurename~\ref{fig:cnn}. We trained an isolation forest using the latent representations of the validation set. In the experiment, we used the \texttt{fit} method of \texttt{sklearn.ensemble.IsolationForest} with default configurations. After the isolation forest is trained, we fed each training input $x$ into the isolation forest and obtained its outlierness score $a$ using the \texttt{score\_samples} method.
\end{itemize}
We also adopted random data removal as the baseline.

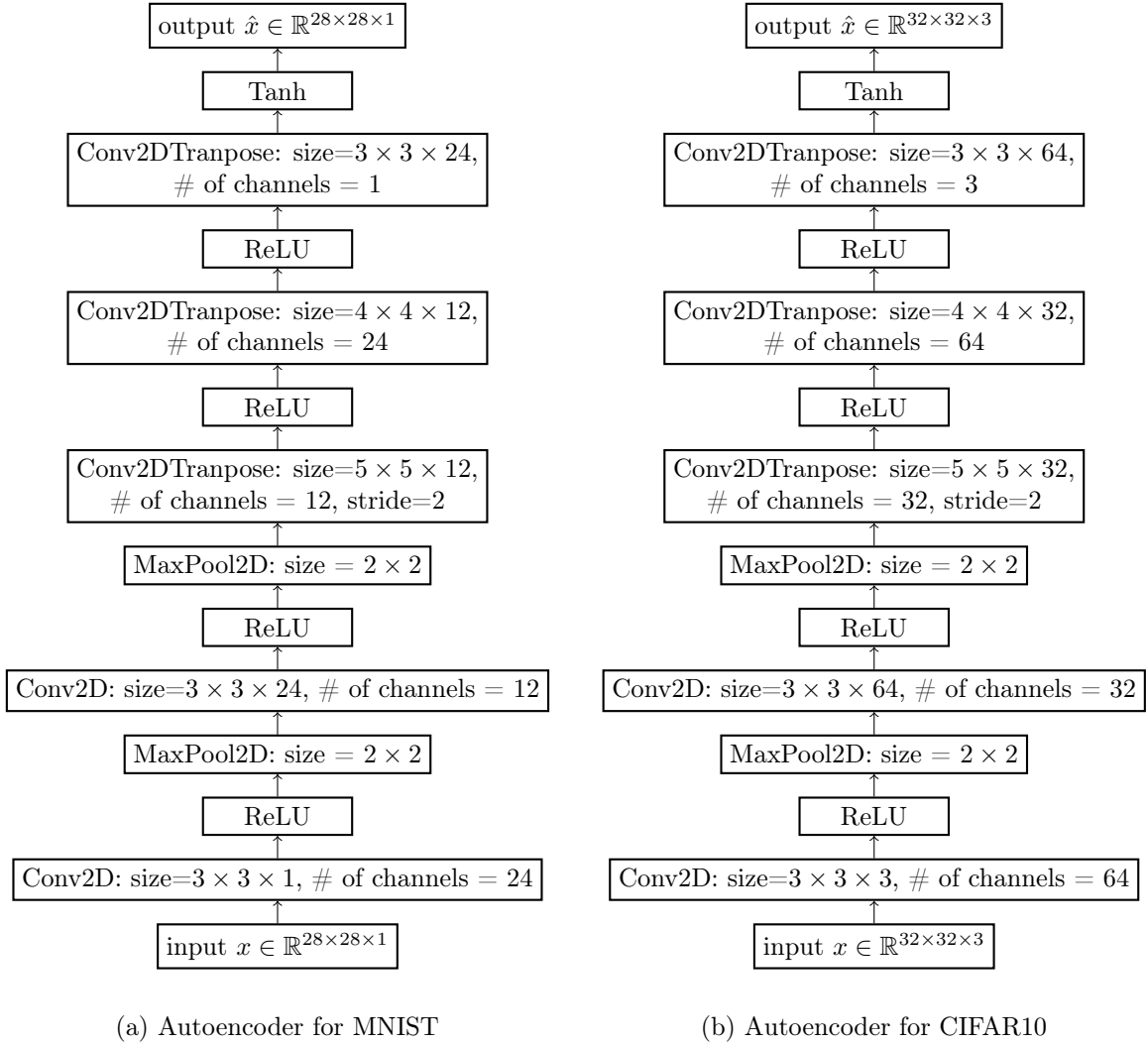
\begin{figure}[t]
    \centering
    \begin{tikzpicture}
        \node (input) at (0,0) [draw,thick,minimum width=2cm,minimum height=0.5cm] {input $x \in \mathbb{R}^{28 \times 28 \times 1}$};
        \node (conv1) [draw,thick,minimum width=2cm,minimum height=0.5cm, above=0.3cm of input] {Conv2D: size=$3 \times 3 \times 1$, \# of channels = $24$};
        \node (relu1) [draw,thick,minimum width=2cm,minimum height=0.5cm, above=0.3cm of conv1] {ReLU};
        \node (pool1) [draw,thick,minimum width=2cm,minimum height=0.5cm, above=0.3cm of relu1] {MaxPool2D: size = $2 \times 2$};
        \node (conv2) [draw,thick,minimum width=2cm,minimum height=0.5cm, above=0.3cm of pool1] {Conv2D: size=$3 \times 3 \times 24$, \# of channels = $12$};
        \node (relu2) [draw,thick,minimum width=2cm,minimum height=0.5cm, above=0.3cm of conv2] {ReLU};
        \node (pool2) [draw,thick,minimum width=2cm,minimum height=0.5cm, above=0.3cm of relu2] {MaxPool2D: size = $2 \times 2$};
        \node (conv2dec) [draw,thick,minimum width=2cm,minimum height=0.5cm, above=0.3cm of pool2, align=center, align=center] {Conv2DTranpose: size=$5 \times 5 \times 12$, \\ \# of channels = $12$, stride=$2$};
        \node (relu2dec) [draw,thick,minimum width=2cm,minimum height=0.5cm, above=0.3cm of conv2dec] {ReLU};
        \node (conv1dec) [draw,thick,minimum width=2cm,minimum height=0.5cm, above=0.3cm of relu2dec, align=center] {Conv2DTranpose: size=$4 \times 4 \times 12$, \\ \# of channels = $24$};
        \node (relu1dec) [draw,thick,minimum width=2cm,minimum height=0.5cm, above=0.3cm of conv1dec] {ReLU};
        \node (conv0dec) [draw,thick,minimum width=2cm,minimum height=0.5cm, above=0.3cm of relu1dec, align=center] {Conv2DTranpose: size=$3 \times 3 \times 24$, \\ \# of channels = $1$};
        \node (tanh) [draw,thick,minimum width=2cm,minimum height=0.5cm, above=0.3cm of conv0dec] {Tanh};
        \node (output) [draw,thick,minimum width=2cm,minimum height=0.5cm, above=0.3cm of tanh] {output $\hat{x} \in \mathbb{R}^{28 \times 28 \times 1}$};
        \draw [->] (input.north) -- (conv1.south);
        \draw [->] (conv1.north) -- (relu1.south);
        \draw [->] (relu1.north) -- (pool1.south);
        \draw [->] (pool1.north) -- (conv2.south);
        \draw [->] (conv2.north) -- (relu2.south);
        \draw [->] (relu2.north) -- (pool2.south);
        \draw [->] (pool2.north) -- (conv2dec.south);
        \draw [->] (conv2dec.north) -- (relu2dec.south);
        \draw [->] (relu2dec.north) -- (conv1dec.south);
        \draw [->] (conv1dec.north) -- (relu1dec.south);
        \draw [->] (relu1dec.north) -- (conv0dec.south);
        \draw [->] (conv0dec.north) -- (tanh.south);
        \draw [->] (tanh.north) -- (output.south);
        \node [below=0.5cm of input.south] {(a) Autoencoder for MNIST};
    %\end{tikzpicture}
    %\begin{tikzpicture}
        \node (input) at (8,0) [draw,thick,minimum width=2cm,minimum height=0.5cm] {input $x \in \mathbb{R}^{32 \times 32 \times 3}$};
        \node (conv1) [draw,thick,minimum width=2cm,minimum height=0.5cm, above=0.3cm of input] {Conv2D: size=$3 \times 3 \times 3$, \# of channels = $64$};
        \node (relu1) [draw,thick,minimum width=2cm,minimum height=0.5cm, above=0.3cm of conv1] {ReLU};
        \node (pool1) [draw,thick,minimum width=2cm,minimum height=0.5cm, above=0.3cm of relu1] {MaxPool2D: size = $2 \times 2$};
        \node (conv2) [draw,thick,minimum width=2cm,minimum height=0.5cm, above=0.3cm of pool1] {Conv2D: size=$3 \times 3 \times 64$, \# of channels = $32$};
        \node (relu2) [draw,thick,minimum width=2cm,minimum height=0.5cm, above=0.3cm of conv2] {ReLU};
        \node (pool2) [draw,thick,minimum width=2cm,minimum height=0.5cm, above=0.3cm of relu2] {MaxPool2D: size = $2 \times 2$};
        \node (conv2dec) [draw,thick,minimum width=2cm,minimum height=0.5cm, above=0.3cm of pool2, align=center] {Conv2DTranpose: size=$5 \times 5 \times 32$, \\ \# of channels = $32$,  stride=$2$};
        \node (relu2dec) [draw,thick,minimum width=2cm,minimum height=0.5cm, above=0.3cm of conv2dec] {ReLU};
        \node (conv1dec) [draw,thick,minimum width=2cm,minimum height=0.5cm, above=0.3cm of relu2dec, align=center] {Conv2DTranpose: size=$4 \times 4 \times 32$, \\ \# of channels = $64$};
        \node (relu1dec) [draw,thick,minimum width=2cm,minimum height=0.5cm, above=0.3cm of conv1dec] {ReLU};
        \node (conv0dec) [draw,thick,minimum width=2cm,minimum height=0.5cm, above=0.3cm of relu1dec, align=center] {Conv2DTranpose: size=$3 \times 3 \times 64$, \\ \# of channels = $3$};
        \node (tanh) [draw,thick,minimum width=2cm,minimum height=0.5cm, above=0.3cm of conv0dec] {Tanh};
        \node (output) [draw,thick,minimum width=2cm,minimum height=0.5cm, above=0.3cm of tanh] {output $\hat{x} \in \mathbb{R}^{32 \times 32 \times 3}$};
        \draw [->] (input.north) -- (conv1.south);
        \draw [->] (conv1.north) -- (relu1.south);
        \draw [->] (relu1.north) -- (pool1.south);
        \draw [->] (pool1.north) -- (conv2.south);
        \draw [->] (conv2.north) -- (relu2.south);
        \draw [->] (relu2.north) -- (pool2.south);
        \draw [->] (pool2.north) -- (conv2dec.south);
        \draw [->] (conv2dec.north) -- (relu2dec.south);
        \draw [->] (relu2dec.north) -- (conv1dec.south);
        \draw [->] (conv1dec.north) -- (relu1dec.south);
        \draw [->] (relu1dec.north) -- (conv0dec.south);
        \draw [->] (conv0dec.north) -- (tanh.south);
        \draw [->] (tanh.north) -- (output.south);
        \node [below=0.5cm of input.south] {(b) Autoencoder for CIFAR10};
    \end{tikzpicture}
    \caption{Structures of Autoencoders}
    \label{fig:ae}
\end{figure}

\paragraph{Proposed Method}
For the proposed method, we introduced an approximate version in this experiment.
In Algorithm~\ref{alg:lie_infer}, the proposed method retraces the entire SGD steps.
In the approximate version, we retrace only one epoch, which requires less computation than the original algorithm.
Moreover, it is also storage friendly because we need to store intermediate information only in the last epoch of SGD.

\paragraph{Procedure}
We proceeded the experiment as follows.
First, we trained the model with SGD using the training set.
We then computed the influence of each training instance using the proposed method as well as the other baseline methods.
Finally, we removed the top-$m$ influential instances from the training set and retrained the model.
For the model ratraining, we considered the two settings.
\begin{itemize}
    \item Retrain All: In this setting, we ran counterfactual SGD for all the 20 epochs with influential instances omitted.
    \item Retrain Last: In this setting, we ran normal SGD for 19 epochs and switched to counterfactual SGD in the last epoch with influential instances omitted.
\end{itemize}
If the misclassification rate of the retrained model decreases, we can conclude that the training set was effectively cleansed.

%%%%%%%%%%%%%%%%%%%%%%%%%%%%%%%%%%%%%%%%%%%%%%%%%%%%%%%%%%%%%%%%%%%%
\clearpage
\subsection{Full Results in Section \ref{sec:exp_cleans}}
\label{app:exp_cleans_res}

The full results for MNIST are shown in Figures~\ref{fig:app_mis_mnist} and \ref{fig:app_mis_mnist_exhaustive}.
The full results for CIFAR10 are shown in Figures \ref{fig:app_mis_cifar10} and \ref{fig:app_mis_cifar10_exhaustive}.
In the figures, it is evident that the misclassification rates have decreased after data cleansing with the proposed method and its approximate version.
We compared the misclassification rates before and after the data cleansing using t-test with the significance level set to $0.05$.
We observed that none of the baseline methods except K\&L attained statistically significant improvements.
By contrast, the proposed method and its approximate version attained statistically significant improvements.
For both datasets, the proposed method and its approximate version were found to be statistically significant for the number of removed instances between 10 and 1000, and 10 and 100, respectively.
Moreover, both methods outperformed K\&L.
\figurename~\ref{fig:compare} also confirms the effectiveness of the data cleansing with the proposed method.
Out of 30 experiments, the misclassification rates decreased with the proposed method for 25 cases in MNIST, and for 26 cases in CIFAR10.
These results confirm that the proposed method can effectively suggest influential instances for data cleansing.
We also note that the proposed method and its approximation version performed comparably well.
This observation suggests that, in practice, we only need to trace back only one epoch for inferring influential instances, which requires less computation and storing intermediate information only in the last epoch of SGD.

Figures~\ref{fig:example_mnist} and \ref{fig:example_cifar10} show the examples of found influential instances.
An interesting observation is that Autoencoder tended to find images with noisy or vivid backgrounds.
Visually, it seems reasonable to select them as outliers.
However, as we have seen in \figurename~\ref{fig:mis}, removing these outliers did not help improving the models.
On the other hand, the proposed method found images with confusing shapes or backgrounds.
Although they are not strongly visually appealing as outliers, \figurename~\ref{fig:mis} confirms that these instances have high impacts to the models.
These observations indicate that the proposed method could find influential instances, which can be missed even by users with domain knowledge.

\begin{figure}[t]
\begin{tikzpicture}
\centering
\begin{axis}[
name=axis1,
width=0.45\textwidth,
height=120pt,
xmode=log,
xmin=1,
xmax=10000,
ymin=0.007,
ymax=0.011,
xlabel=\# of instances removed,
ylabel=Misclassification rate,
label style={font=\footnotesize},
scaled ticks=false,
tick label style={
	/pgf/number format/fixed,
    /pgf/number format/precision=4,
    font=\scriptsize
},
legend style={font=\fontsize{7}{5}\selectfont},
legend style={at={(1.1,1.08)},anchor=south,legend columns=-1}
]
\addplot[black, densely dotted, line width=0.5mm] table[x=k, y=baseline, col sep=comma]{./figs/mnist_miss.csv};
\addplot[magenta, mark=diamond*, line width=0.3mm] table[x=k, y=random, col sep=comma]{./figs/mnist_miss.csv};
\addplot[color={rgb:red,2;green,4;blue,2}, line width=0.3mm] table[x=k, y=ae, col sep=comma]{./figs/mnist_miss.csv};
\addplot[color={rgb:red,2;green,4;blue,2}, densely dotted, line width=0.3mm] table[x=k, y=iso, col sep=comma]{./figs/mnist_miss.csv};
\addplot[blue, mark=*, mark size=1.5pt, line width=0.3mm] table[x=k, y=sgd_all, col sep=comma]{./figs/mnist_miss.csv};
\addplot[blue, mark options=solid, densely dotted, mark=triangle*, mark size=2pt, line width=0.3mm] table[x=k, y=sgd_last, col sep=comma]{./figs/mnist_miss.csv};
\addplot[red, mark=square*, mark size=1.5pt, line width=0.3mm] table[x=k, y=icml, col sep=comma]{./figs/mnist_miss.csv};
\addplot[black, densely dotted, line width=0.5mm] table[x=k, y=baseline, col sep=comma]{./figs/mnist_miss.csv};
\legend{No Removal, Random, Autoencoder, Isolation Forest, Proposed, Proposed (Approx.), K\&L}
\end{axis}
\begin{axis}[
xshift=2.7in,
name=axis2,
width=0.45\textwidth,
height=120pt,
xmode=log,
xmin=1,
xmax=10000,
ymin=0.007,
ymax=0.011,
xlabel=\# of instances removed,
ylabel=Misclassification rate,
label style={font=\footnotesize},
scaled ticks=false,
tick label style={
	/pgf/number format/fixed,
    /pgf/number format/precision=4,
    font=\scriptsize
},
legend style={font=\fontsize{7}{5}\selectfont},
legend pos=north west
]
\addplot[black, densely dotted, line width=0.5mm] table[x=k, y=baseline, col sep=comma]{./figs/mnist_miss_all.csv};
\addplot[magenta, mark=diamond*, line width=0.3mm] table[x=k, y=random, col sep=comma]{./figs/mnist_miss_all.csv};
\addplot[color={rgb:red,2;green,4;blue,2}, line width=0.3mm] table[x=k, y=ae, col sep=comma]{./figs/mnist_miss_all.csv};
\addplot[color={rgb:red,2;green,4;blue,2}, densely dotted, line width=0.3mm] table[x=k, y=iso, col sep=comma]{./figs/mnist_miss_all.csv};
\addplot[blue, mark=*, mark size=1.5pt, line width=0.3mm] table[x=k, y=sgd_all, col sep=comma]{./figs/mnist_miss_all.csv};
\addplot[blue, mark options=solid, densely dotted, mark=triangle*, mark size=2pt, line width=0.3mm] table[x=k, y=sgd_last, col sep=comma]{./figs/mnist_miss_all.csv};
\addplot[red, mark=square*, mark size=1.5pt, line width=0.3mm] table[x=k, y=icml, col sep=comma]{./figs/mnist_miss_all.csv};
\addplot[black, densely dotted, line width=0.5mm] table[x=k, y=baseline, col sep=comma]{./figs/mnist_miss_all.csv};
\end{axis}
\node [below=1cm] at (axis1.south) {\footnotesize (a) MNIST: Retrain Last};
\node [below=1cm] at (axis2.south) {\footnotesize (b) MNIST: Retrain All};
\end{tikzpicture}
\caption{MNIST: Average misclassification rates on the test set after data cleansing over 30 experiments.}
\label{fig:app_mis_mnist}
%\end{figure}
%
%\begin{figure}
\vspace{12pt}
\centering
\begin{tikzpicture}
\node [anchor=center] at (0,0) [draw,minimum width=9cm,minimum height=0.5cm] {};
\node at (-2.5, 0) {\footnotesize No Removal};
\draw[dashed] (-4.0, 0) -- (-3.5, 0);
\node at (0.5, 0) {\footnotesize Retrain All};
\draw[red] (-1.0, 0) -- (-0.5, 0);
\fill[red] (-0.75,0.12) -- (-0.65,-0.06) -- (-0.85,-0.06) -- cycle;
\node at (3.5, 0) {\footnotesize Retrain Last};
\draw[blue] (2.0, 0) -- (2.5, 0);
\fill [blue] (2.25, 0) circle [radius=0.08];
\end{tikzpicture}
\\
  \def\xs{0}%
  \def\ys{0}%
  \def\key{random}%
  \def\leg{(a) Random}%
  \begin{tikzpicture}
\begin{axis}[
name=axis,
xshift=\xs cm,
yshift=\ys cm,
width=0.45\textwidth,
height=120pt,
xmode=log,
xmin=1,
xmax=10000,
ymin=0.006,
ymax=0.012,
xlabel=\# of instances removed,
ylabel=Misclassification rate,
label style={font=\footnotesize},
scaled ticks=false,
tick label style={
	/pgf/number format/fixed,
    /pgf/number format/precision=4,
    font=\scriptsize
},
legend style={font=\fontsize{7}{5}\selectfont},
legend style={at={(0,1.08)},anchor=south,legend columns=-1}
]
\addplot[name path=b_upper,draw=none,forget plot] table[x=k,y expr=\thisrow{baseline}+\thisrow{baseline_std}, col sep=comma] {./figs/mnist_miss.csv};
\addplot[name path=b_lower,draw=none,forget plot] table[x=k,y expr=\thisrow{baseline}-\thisrow{baseline_std}, col sep=comma] {./figs/mnist_miss.csv};
\addplot [fill=black!30,opacity=0.3,forget plot] fill between[of=b_upper and b_lower];

\addplot[name path=all_upper,draw=none,forget plot] table[x=k,y expr=\thisrow{\key}+\thisrow{\key_std}, col sep=comma] {./figs/mnist_miss_all.csv};
\addplot[name path=all_lower,draw=none,forget plot] table[x=k,y expr=\thisrow{\key}-\thisrow{\key_std}, col sep=comma] {./figs/mnist_miss_all.csv};
\addplot [fill=red!30,opacity=0.7,forget plot] fill between[of=all_upper and all_lower];

\addplot[name path=last_upper,draw=none,forget plot] table[x=k,y expr=\thisrow{\key}+\thisrow{\key_std}, col sep=comma] {./figs/mnist_miss.csv};
\addplot[name path=last_lower,draw=none,forget plot] table[x=k,y expr=\thisrow{\key}-\thisrow{\key_std}, col sep=comma] {./figs/mnist_miss.csv};
\addplot [fill=blue!30,opacity=0.5,forget plot] fill between[of=last_upper and last_lower];

\addplot[black, dashed] table[x=k, y=baseline, col sep=comma]{./figs/mnist_miss.csv};
\addplot[red, mark=triangle*, mark size=3pt] table[x=k, y=\key, col sep=comma]{./figs/mnist_miss_all.csv};
\addplot[blue, mark=*, mark size=2pt] table[x=k, y=\key, col sep=comma]{./figs/mnist_miss.csv};
\end{axis}
\node[below=1.1cm of axis.south] {\leg};
\end{tikzpicture}%

  \def\xs{0}%
  \def\ys{0}%
  \def\key{icml}%
  \def\leg{(b) K\&L}%
  \begin{tikzpicture}
\begin{axis}[
name=axis,
xshift=\xs cm,
yshift=\ys cm,
width=0.45\textwidth,
height=120pt,
xmode=log,
xmin=1,
xmax=10000,
ymin=0.006,
ymax=0.012,
xlabel=\# of instances removed,
ylabel=Misclassification rate,
label style={font=\footnotesize},
scaled ticks=false,
tick label style={
	/pgf/number format/fixed,
    /pgf/number format/precision=4,
    font=\scriptsize
},
legend style={font=\fontsize{7}{5}\selectfont},
legend style={at={(0,1.08)},anchor=south,legend columns=-1}
]
\addplot[name path=b_upper,draw=none,forget plot] table[x=k,y expr=\thisrow{baseline}+\thisrow{baseline_std}, col sep=comma] {./figs/mnist_miss.csv};
\addplot[name path=b_lower,draw=none,forget plot] table[x=k,y expr=\thisrow{baseline}-\thisrow{baseline_std}, col sep=comma] {./figs/mnist_miss.csv};
\addplot [fill=black!30,opacity=0.3,forget plot] fill between[of=b_upper and b_lower];

\addplot[name path=all_upper,draw=none,forget plot] table[x=k,y expr=\thisrow{\key}+\thisrow{\key_std}, col sep=comma] {./figs/mnist_miss_all.csv};
\addplot[name path=all_lower,draw=none,forget plot] table[x=k,y expr=\thisrow{\key}-\thisrow{\key_std}, col sep=comma] {./figs/mnist_miss_all.csv};
\addplot [fill=red!30,opacity=0.7,forget plot] fill between[of=all_upper and all_lower];

\addplot[name path=last_upper,draw=none,forget plot] table[x=k,y expr=\thisrow{\key}+\thisrow{\key_std}, col sep=comma] {./figs/mnist_miss.csv};
\addplot[name path=last_lower,draw=none,forget plot] table[x=k,y expr=\thisrow{\key}-\thisrow{\key_std}, col sep=comma] {./figs/mnist_miss.csv};
\addplot [fill=blue!30,opacity=0.5,forget plot] fill between[of=last_upper and last_lower];

\addplot[black, dashed] table[x=k, y=baseline, col sep=comma]{./figs/mnist_miss.csv};
\addplot[red, mark=triangle*, mark size=3pt] table[x=k, y=\key, col sep=comma]{./figs/mnist_miss_all.csv};
\addplot[blue, mark=*, mark size=2pt] table[x=k, y=\key, col sep=comma]{./figs/mnist_miss.csv};
\end{axis}
\node[below=1.1cm of axis.south] {\leg};
\end{tikzpicture}%
\\
  \def\xs{0}%
  \def\ys{0}%
  \def\key{ae}%
  \def\leg{(c) Autoencoder}%
  \begin{tikzpicture}
\begin{axis}[
name=axis,
xshift=\xs cm,
yshift=\ys cm,
width=0.45\textwidth,
height=120pt,
xmode=log,
xmin=1,
xmax=10000,
ymin=0.006,
ymax=0.012,
xlabel=\# of instances removed,
ylabel=Misclassification rate,
label style={font=\footnotesize},
scaled ticks=false,
tick label style={
	/pgf/number format/fixed,
    /pgf/number format/precision=4,
    font=\scriptsize
},
legend style={font=\fontsize{7}{5}\selectfont},
legend style={at={(0,1.08)},anchor=south,legend columns=-1}
]
\addplot[name path=b_upper,draw=none,forget plot] table[x=k,y expr=\thisrow{baseline}+\thisrow{baseline_std}, col sep=comma] {./figs/mnist_miss.csv};
\addplot[name path=b_lower,draw=none,forget plot] table[x=k,y expr=\thisrow{baseline}-\thisrow{baseline_std}, col sep=comma] {./figs/mnist_miss.csv};
\addplot [fill=black!30,opacity=0.3,forget plot] fill between[of=b_upper and b_lower];

\addplot[name path=all_upper,draw=none,forget plot] table[x=k,y expr=\thisrow{\key}+\thisrow{\key_std}, col sep=comma] {./figs/mnist_miss_all.csv};
\addplot[name path=all_lower,draw=none,forget plot] table[x=k,y expr=\thisrow{\key}-\thisrow{\key_std}, col sep=comma] {./figs/mnist_miss_all.csv};
\addplot [fill=red!30,opacity=0.7,forget plot] fill between[of=all_upper and all_lower];

\addplot[name path=last_upper,draw=none,forget plot] table[x=k,y expr=\thisrow{\key}+\thisrow{\key_std}, col sep=comma] {./figs/mnist_miss.csv};
\addplot[name path=last_lower,draw=none,forget plot] table[x=k,y expr=\thisrow{\key}-\thisrow{\key_std}, col sep=comma] {./figs/mnist_miss.csv};
\addplot [fill=blue!30,opacity=0.5,forget plot] fill between[of=last_upper and last_lower];

\addplot[black, dashed] table[x=k, y=baseline, col sep=comma]{./figs/mnist_miss.csv};
\addplot[red, mark=triangle*, mark size=3pt] table[x=k, y=\key, col sep=comma]{./figs/mnist_miss_all.csv};
\addplot[blue, mark=*, mark size=2pt] table[x=k, y=\key, col sep=comma]{./figs/mnist_miss.csv};
\end{axis}
\node[below=1.1cm of axis.south] {\leg};
\end{tikzpicture}%

  \def\xs{0}%
  \def\ys{0}%
  \def\key{iso}%
  \def\leg{(d) Isolation Forest}%
  \begin{tikzpicture}
\begin{axis}[
name=axis,
xshift=\xs cm,
yshift=\ys cm,
width=0.45\textwidth,
height=120pt,
xmode=log,
xmin=1,
xmax=10000,
ymin=0.006,
ymax=0.012,
xlabel=\# of instances removed,
ylabel=Misclassification rate,
label style={font=\footnotesize},
scaled ticks=false,
tick label style={
	/pgf/number format/fixed,
    /pgf/number format/precision=4,
    font=\scriptsize
},
legend style={font=\fontsize{7}{5}\selectfont},
legend style={at={(0,1.08)},anchor=south,legend columns=-1}
]
\addplot[name path=b_upper,draw=none,forget plot] table[x=k,y expr=\thisrow{baseline}+\thisrow{baseline_std}, col sep=comma] {./figs/mnist_miss.csv};
\addplot[name path=b_lower,draw=none,forget plot] table[x=k,y expr=\thisrow{baseline}-\thisrow{baseline_std}, col sep=comma] {./figs/mnist_miss.csv};
\addplot [fill=black!30,opacity=0.3,forget plot] fill between[of=b_upper and b_lower];

\addplot[name path=all_upper,draw=none,forget plot] table[x=k,y expr=\thisrow{\key}+\thisrow{\key_std}, col sep=comma] {./figs/mnist_miss_all.csv};
\addplot[name path=all_lower,draw=none,forget plot] table[x=k,y expr=\thisrow{\key}-\thisrow{\key_std}, col sep=comma] {./figs/mnist_miss_all.csv};
\addplot [fill=red!30,opacity=0.7,forget plot] fill between[of=all_upper and all_lower];

\addplot[name path=last_upper,draw=none,forget plot] table[x=k,y expr=\thisrow{\key}+\thisrow{\key_std}, col sep=comma] {./figs/mnist_miss.csv};
\addplot[name path=last_lower,draw=none,forget plot] table[x=k,y expr=\thisrow{\key}-\thisrow{\key_std}, col sep=comma] {./figs/mnist_miss.csv};
\addplot [fill=blue!30,opacity=0.5,forget plot] fill between[of=last_upper and last_lower];

\addplot[black, dashed] table[x=k, y=baseline, col sep=comma]{./figs/mnist_miss.csv};
\addplot[red, mark=triangle*, mark size=3pt] table[x=k, y=\key, col sep=comma]{./figs/mnist_miss_all.csv};
\addplot[blue, mark=*, mark size=2pt] table[x=k, y=\key, col sep=comma]{./figs/mnist_miss.csv};
\end{axis}
\node[below=1.1cm of axis.south] {\leg};
\end{tikzpicture}%
\\
  \def\xs{0}%
  \def\ys{0}%
  \def\key{sgd_all}%
  \def\leg{(e) Proposed}%
  \begin{tikzpicture}
\begin{axis}[
name=axis,
xshift=\xs cm,
yshift=\ys cm,
width=0.45\textwidth,
height=120pt,
xmode=log,
xmin=1,
xmax=10000,
ymin=0.006,
ymax=0.012,
xlabel=\# of instances removed,
ylabel=Misclassification rate,
label style={font=\footnotesize},
scaled ticks=false,
tick label style={
	/pgf/number format/fixed,
    /pgf/number format/precision=4,
    font=\scriptsize
},
legend style={font=\fontsize{7}{5}\selectfont},
legend style={at={(0,1.08)},anchor=south,legend columns=-1}
]
\addplot[name path=b_upper,draw=none,forget plot] table[x=k,y expr=\thisrow{baseline}+\thisrow{baseline_std}, col sep=comma] {./figs/mnist_miss.csv};
\addplot[name path=b_lower,draw=none,forget plot] table[x=k,y expr=\thisrow{baseline}-\thisrow{baseline_std}, col sep=comma] {./figs/mnist_miss.csv};
\addplot [fill=black!30,opacity=0.3,forget plot] fill between[of=b_upper and b_lower];

\addplot[name path=all_upper,draw=none,forget plot] table[x=k,y expr=\thisrow{\key}+\thisrow{\key_std}, col sep=comma] {./figs/mnist_miss_all.csv};
\addplot[name path=all_lower,draw=none,forget plot] table[x=k,y expr=\thisrow{\key}-\thisrow{\key_std}, col sep=comma] {./figs/mnist_miss_all.csv};
\addplot [fill=red!30,opacity=0.7,forget plot] fill between[of=all_upper and all_lower];

\addplot[name path=last_upper,draw=none,forget plot] table[x=k,y expr=\thisrow{\key}+\thisrow{\key_std}, col sep=comma] {./figs/mnist_miss.csv};
\addplot[name path=last_lower,draw=none,forget plot] table[x=k,y expr=\thisrow{\key}-\thisrow{\key_std}, col sep=comma] {./figs/mnist_miss.csv};
\addplot [fill=blue!30,opacity=0.5,forget plot] fill between[of=last_upper and last_lower];

\addplot[black, dashed] table[x=k, y=baseline, col sep=comma]{./figs/mnist_miss.csv};
\addplot[red, mark=triangle*, mark size=3pt] table[x=k, y=\key, col sep=comma]{./figs/mnist_miss_all.csv};
\addplot[blue, mark=*, mark size=2pt] table[x=k, y=\key, col sep=comma]{./figs/mnist_miss.csv};
\end{axis}
\node[below=1.1cm of axis.south] {\leg};
\end{tikzpicture}%

  \def\xs{0}%
  \def\ys{0}%
  \def\key{sgd_last}%
  \def\leg{(f) Proposed (Approx.)}%
  \begin{tikzpicture}
\begin{axis}[
name=axis,
xshift=\xs cm,
yshift=\ys cm,
width=0.45\textwidth,
height=120pt,
xmode=log,
xmin=1,
xmax=10000,
ymin=0.006,
ymax=0.012,
xlabel=\# of instances removed,
ylabel=Misclassification rate,
label style={font=\footnotesize},
scaled ticks=false,
tick label style={
	/pgf/number format/fixed,
    /pgf/number format/precision=4,
    font=\scriptsize
},
legend style={font=\fontsize{7}{5}\selectfont},
legend style={at={(0,1.08)},anchor=south,legend columns=-1}
]
\addplot[name path=b_upper,draw=none,forget plot] table[x=k,y expr=\thisrow{baseline}+\thisrow{baseline_std}, col sep=comma] {./figs/mnist_miss.csv};
\addplot[name path=b_lower,draw=none,forget plot] table[x=k,y expr=\thisrow{baseline}-\thisrow{baseline_std}, col sep=comma] {./figs/mnist_miss.csv};
\addplot [fill=black!30,opacity=0.3,forget plot] fill between[of=b_upper and b_lower];

\addplot[name path=all_upper,draw=none,forget plot] table[x=k,y expr=\thisrow{\key}+\thisrow{\key_std}, col sep=comma] {./figs/mnist_miss_all.csv};
\addplot[name path=all_lower,draw=none,forget plot] table[x=k,y expr=\thisrow{\key}-\thisrow{\key_std}, col sep=comma] {./figs/mnist_miss_all.csv};
\addplot [fill=red!30,opacity=0.7,forget plot] fill between[of=all_upper and all_lower];

\addplot[name path=last_upper,draw=none,forget plot] table[x=k,y expr=\thisrow{\key}+\thisrow{\key_std}, col sep=comma] {./figs/mnist_miss.csv};
\addplot[name path=last_lower,draw=none,forget plot] table[x=k,y expr=\thisrow{\key}-\thisrow{\key_std}, col sep=comma] {./figs/mnist_miss.csv};
\addplot [fill=blue!30,opacity=0.5,forget plot] fill between[of=last_upper and last_lower];

\addplot[black, dashed] table[x=k, y=baseline, col sep=comma]{./figs/mnist_miss.csv};
\addplot[red, mark=triangle*, mark size=3pt] table[x=k, y=\key, col sep=comma]{./figs/mnist_miss_all.csv};
\addplot[blue, mark=*, mark size=2pt] table[x=k, y=\key, col sep=comma]{./figs/mnist_miss.csv};
\end{axis}
\node[below=1.1cm of axis.south] {\leg};
\end{tikzpicture}%

\caption{Exhaustive results on MNIST: [Thick lines] Average misclassification rates on the test set after data cleansing over 30 experiments. [Shaded Regions] Average $\pm$ standard deviation.}
\label{fig:app_mis_mnist_exhaustive}
\end{figure}

\begin{figure}[t]
\begin{tikzpicture}
\centering
\begin{axis}[
name=axis1,
width=0.45\textwidth,
height=120pt,
xmode=log,
xmin=1,
xmax=10000,
ymin=0.15,
ymax=0.19,
xlabel=\# of instances removed,
ylabel=Misclassification rate,
label style={font=\footnotesize},
scaled ticks=false,
tick label style={
	/pgf/number format/fixed,
    /pgf/number format/precision=4,
    font=\scriptsize
},
legend style={font=\fontsize{7}{5}\selectfont},
legend style={at={(1.1,1.08)},anchor=south,legend columns=-1}
]
\addplot[black, densely dotted, line width=0.5mm] table[x=k, y=baseline, col sep=comma]{./figs/cifar10_miss.csv};
\addplot[magenta, mark=diamond*, line width=0.3mm] table[x=k, y=random, col sep=comma]{./figs/cifar10_miss.csv};
\addplot[color={rgb:red,2;green,4;blue,2}, line width=0.3mm] table[x=k, y=ae, col sep=comma]{./figs/cifar10_miss.csv};
\addplot[color={rgb:red,2;green,4;blue,2}, densely dotted, line width=0.3mm] table[x=k, y=iso, col sep=comma]{./figs/cifar10_miss.csv};
\addplot[blue, mark=*, mark size=1.5pt, line width=0.3mm] table[x=k, y=sgd_all, col sep=comma]{./figs/cifar10_miss.csv};
\addplot[blue, mark options=solid, densely dotted, mark=triangle*, mark size=2pt, line width=0.3mm] table[x=k, y=sgd_last, col sep=comma]{./figs/cifar10_miss.csv};
\addplot[red, mark=square*, mark size=1.5pt, line width=0.3mm] table[x=k, y=icml, col sep=comma]{./figs/cifar10_miss.csv};
\addplot[black, densely dotted, line width=0.5mm] table[x=k, y=baseline, col sep=comma]{./figs/cifar10_miss.csv};
\legend{No Removal, Random, Autoencoder, Isolation Forest, Proposed, Proposed (Approx.), K\&L}
\end{axis}
\begin{axis}[
xshift=2.7in,
name=axis2,
width=0.45\textwidth,
height=120pt,
xmode=log,
xmin=1,
xmax=10000,
ymin=0.15,
ymax=0.19,
xlabel=\# of instances removed,
ylabel=Misclassification rate,
label style={font=\footnotesize},
scaled ticks=false,
tick label style={
	/pgf/number format/fixed,
    /pgf/number format/precision=4,
    font=\scriptsize
},
legend style={font=\fontsize{7}{5}\selectfont},
legend pos=north west
]
\addplot[black, densely dotted, line width=0.5mm] table[x=k, y=baseline, col sep=comma]{./figs/cifar10_miss_all.csv};
\addplot[magenta, mark=diamond*, line width=0.3mm] table[x=k, y=random, col sep=comma]{./figs/cifar10_miss_all.csv};
\addplot[color={rgb:red,2;green,4;blue,2}, line width=0.3mm] table[x=k, y=ae, col sep=comma]{./figs/cifar10_miss_all.csv};
\addplot[color={rgb:red,2;green,4;blue,2}, densely dotted, line width=0.3mm] table[x=k, y=iso, col sep=comma]{./figs/cifar10_miss_all.csv};
\addplot[blue, mark=*, mark size=1.5pt, line width=0.3mm] table[x=k, y=sgd_all, col sep=comma]{./figs/cifar10_miss_all.csv};
\addplot[blue, mark options=solid, densely dotted, mark=triangle*, mark size=2pt, line width=0.3mm] table[x=k, y=sgd_last, col sep=comma]{./figs/cifar10_miss_all.csv};
\addplot[red, mark=square*, mark size=1.5pt, line width=0.3mm] table[x=k, y=icml, col sep=comma]{./figs/cifar10_miss_all.csv};
\addplot[black, densely dotted, line width=0.5mm] table[x=k, y=baseline, col sep=comma]{./figs/cifar10_miss_all.csv};
\end{axis}
\node [below=1cm] at (axis1.south) {\footnotesize (a) CIFAR10: Retrain Last};
\node [below=1cm] at (axis2.south) {\footnotesize (b) CIFAR10: Retrain All};
\end{tikzpicture}
\caption{CIFAR10: Average misclassification rates on the test set after data cleansing over 30 experiments.}
\label{fig:app_mis_cifar10}
%\end{figure}
%
%\begin{figure}
\vspace{12pt}
\centering
\begin{tikzpicture}
\node [anchor=center] at (0,0) [draw,minimum width=9cm,minimum height=0.5cm] {};
\node at (-2.5, 0) {\footnotesize No Removal};
\draw[dashed] (-4.0, 0) -- (-3.5, 0);
\node at (0.5, 0) {\footnotesize Retrain All};
\draw[red] (-1.0, 0) -- (-0.5, 0);
\fill[red] (-0.75,0.12) -- (-0.65,-0.06) -- (-0.85,-0.06) -- cycle;
\node at (3.5, 0) {\footnotesize Retrain Last};
\draw[blue] (2.0, 0) -- (2.5, 0);
\fill [blue] (2.25, 0) circle [radius=0.08];
\end{tikzpicture}
\\
  \def\xs{0}%
  \def\ys{0}%
  \def\key{random}%
  \def\leg{(a) Random}%
  \begin{tikzpicture}
\begin{axis}[
name=axis,
xshift=\xs cm,
yshift=\ys cm,
width=0.45\textwidth,
height=120pt,
xmode=log,
xmin=1,
xmax=10000,
ymin=0.145,
ymax=0.19,
xlabel=\# of instances removed,
ylabel=Misclassification rate,
label style={font=\footnotesize},
scaled ticks=false,
tick label style={
	/pgf/number format/fixed,
    /pgf/number format/precision=4,
    font=\scriptsize
},
legend style={font=\fontsize{7}{5}\selectfont},
legend style={at={(0,1.08)},anchor=south,legend columns=-1}
]
\addplot[name path=b_upper,draw=none,forget plot] table[x=k,y expr=\thisrow{baseline}+\thisrow{baseline_std}, col sep=comma] {./figs/cifar10_miss.csv};
\addplot[name path=b_lower,draw=none,forget plot] table[x=k,y expr=\thisrow{baseline}-\thisrow{baseline_std}, col sep=comma] {./figs/cifar10_miss.csv};
\addplot [fill=black!30,opacity=0.3,forget plot] fill between[of=b_upper and b_lower];

\addplot[name path=all_upper,draw=none,forget plot] table[x=k,y expr=\thisrow{\key}+\thisrow{\key_std}, col sep=comma] {./figs/cifar10_miss_all.csv};
\addplot[name path=all_lower,draw=none,forget plot] table[x=k,y expr=\thisrow{\key}-\thisrow{\key_std}, col sep=comma] {./figs/cifar10_miss_all.csv};
\addplot [fill=red!30,opacity=0.7,forget plot] fill between[of=all_upper and all_lower];

\addplot[name path=last_upper,draw=none,forget plot] table[x=k,y expr=\thisrow{\key}+\thisrow{\key_std}, col sep=comma] {./figs/cifar10_miss.csv};
\addplot[name path=last_lower,draw=none,forget plot] table[x=k,y expr=\thisrow{\key}-\thisrow{\key_std}, col sep=comma] {./figs/cifar10_miss.csv};
\addplot [fill=blue!30,opacity=0.5,forget plot] fill between[of=last_upper and last_lower];

\addplot[black, dashed] table[x=k, y=baseline, col sep=comma]{./figs/cifar10_miss.csv};
\addplot[red, mark=triangle*, mark size=3pt] table[x=k, y=\key, col sep=comma]{./figs/cifar10_miss_all.csv};
\addplot[blue, mark=*, mark size=2pt] table[x=k, y=\key, col sep=comma]{./figs/cifar10_miss.csv};
\end{axis}
\node[below=1.1cm of axis.south] {\leg};
\end{tikzpicture}%

  \def\xs{0}%
  \def\ys{0}%
  \def\key{icml}%
  \def\leg{(b) K\&L}%
  \begin{tikzpicture}
\begin{axis}[
name=axis,
xshift=\xs cm,
yshift=\ys cm,
width=0.45\textwidth,
height=120pt,
xmode=log,
xmin=1,
xmax=10000,
ymin=0.145,
ymax=0.19,
xlabel=\# of instances removed,
ylabel=Misclassification rate,
label style={font=\footnotesize},
scaled ticks=false,
tick label style={
	/pgf/number format/fixed,
    /pgf/number format/precision=4,
    font=\scriptsize
},
legend style={font=\fontsize{7}{5}\selectfont},
legend style={at={(0,1.08)},anchor=south,legend columns=-1}
]
\addplot[name path=b_upper,draw=none,forget plot] table[x=k,y expr=\thisrow{baseline}+\thisrow{baseline_std}, col sep=comma] {./figs/cifar10_miss.csv};
\addplot[name path=b_lower,draw=none,forget plot] table[x=k,y expr=\thisrow{baseline}-\thisrow{baseline_std}, col sep=comma] {./figs/cifar10_miss.csv};
\addplot [fill=black!30,opacity=0.3,forget plot] fill between[of=b_upper and b_lower];

\addplot[name path=all_upper,draw=none,forget plot] table[x=k,y expr=\thisrow{\key}+\thisrow{\key_std}, col sep=comma] {./figs/cifar10_miss_all.csv};
\addplot[name path=all_lower,draw=none,forget plot] table[x=k,y expr=\thisrow{\key}-\thisrow{\key_std}, col sep=comma] {./figs/cifar10_miss_all.csv};
\addplot [fill=red!30,opacity=0.7,forget plot] fill between[of=all_upper and all_lower];

\addplot[name path=last_upper,draw=none,forget plot] table[x=k,y expr=\thisrow{\key}+\thisrow{\key_std}, col sep=comma] {./figs/cifar10_miss.csv};
\addplot[name path=last_lower,draw=none,forget plot] table[x=k,y expr=\thisrow{\key}-\thisrow{\key_std}, col sep=comma] {./figs/cifar10_miss.csv};
\addplot [fill=blue!30,opacity=0.5,forget plot] fill between[of=last_upper and last_lower];

\addplot[black, dashed] table[x=k, y=baseline, col sep=comma]{./figs/cifar10_miss.csv};
\addplot[red, mark=triangle*, mark size=3pt] table[x=k, y=\key, col sep=comma]{./figs/cifar10_miss_all.csv};
\addplot[blue, mark=*, mark size=2pt] table[x=k, y=\key, col sep=comma]{./figs/cifar10_miss.csv};
\end{axis}
\node[below=1.1cm of axis.south] {\leg};
\end{tikzpicture}%
\\
  \def\xs{0}%
  \def\ys{0}%
  \def\key{ae}%
  \def\leg{(c) Autoencoder}%
  \begin{tikzpicture}
\begin{axis}[
name=axis,
xshift=\xs cm,
yshift=\ys cm,
width=0.45\textwidth,
height=120pt,
xmode=log,
xmin=1,
xmax=10000,
ymin=0.145,
ymax=0.19,
xlabel=\# of instances removed,
ylabel=Misclassification rate,
label style={font=\footnotesize},
scaled ticks=false,
tick label style={
	/pgf/number format/fixed,
    /pgf/number format/precision=4,
    font=\scriptsize
},
legend style={font=\fontsize{7}{5}\selectfont},
legend style={at={(0,1.08)},anchor=south,legend columns=-1}
]
\addplot[name path=b_upper,draw=none,forget plot] table[x=k,y expr=\thisrow{baseline}+\thisrow{baseline_std}, col sep=comma] {./figs/cifar10_miss.csv};
\addplot[name path=b_lower,draw=none,forget plot] table[x=k,y expr=\thisrow{baseline}-\thisrow{baseline_std}, col sep=comma] {./figs/cifar10_miss.csv};
\addplot [fill=black!30,opacity=0.3,forget plot] fill between[of=b_upper and b_lower];

\addplot[name path=all_upper,draw=none,forget plot] table[x=k,y expr=\thisrow{\key}+\thisrow{\key_std}, col sep=comma] {./figs/cifar10_miss_all.csv};
\addplot[name path=all_lower,draw=none,forget plot] table[x=k,y expr=\thisrow{\key}-\thisrow{\key_std}, col sep=comma] {./figs/cifar10_miss_all.csv};
\addplot [fill=red!30,opacity=0.7,forget plot] fill between[of=all_upper and all_lower];

\addplot[name path=last_upper,draw=none,forget plot] table[x=k,y expr=\thisrow{\key}+\thisrow{\key_std}, col sep=comma] {./figs/cifar10_miss.csv};
\addplot[name path=last_lower,draw=none,forget plot] table[x=k,y expr=\thisrow{\key}-\thisrow{\key_std}, col sep=comma] {./figs/cifar10_miss.csv};
\addplot [fill=blue!30,opacity=0.5,forget plot] fill between[of=last_upper and last_lower];

\addplot[black, dashed] table[x=k, y=baseline, col sep=comma]{./figs/cifar10_miss.csv};
\addplot[red, mark=triangle*, mark size=3pt] table[x=k, y=\key, col sep=comma]{./figs/cifar10_miss_all.csv};
\addplot[blue, mark=*, mark size=2pt] table[x=k, y=\key, col sep=comma]{./figs/cifar10_miss.csv};
\end{axis}
\node[below=1.1cm of axis.south] {\leg};
\end{tikzpicture}%

  \def\xs{0}%
  \def\ys{0}%
  \def\key{iso}%
  \def\leg{(d) Isolation Forest}%
  \begin{tikzpicture}
\begin{axis}[
name=axis,
xshift=\xs cm,
yshift=\ys cm,
width=0.45\textwidth,
height=120pt,
xmode=log,
xmin=1,
xmax=10000,
ymin=0.145,
ymax=0.19,
xlabel=\# of instances removed,
ylabel=Misclassification rate,
label style={font=\footnotesize},
scaled ticks=false,
tick label style={
	/pgf/number format/fixed,
    /pgf/number format/precision=4,
    font=\scriptsize
},
legend style={font=\fontsize{7}{5}\selectfont},
legend style={at={(0,1.08)},anchor=south,legend columns=-1}
]
\addplot[name path=b_upper,draw=none,forget plot] table[x=k,y expr=\thisrow{baseline}+\thisrow{baseline_std}, col sep=comma] {./figs/cifar10_miss.csv};
\addplot[name path=b_lower,draw=none,forget plot] table[x=k,y expr=\thisrow{baseline}-\thisrow{baseline_std}, col sep=comma] {./figs/cifar10_miss.csv};
\addplot [fill=black!30,opacity=0.3,forget plot] fill between[of=b_upper and b_lower];

\addplot[name path=all_upper,draw=none,forget plot] table[x=k,y expr=\thisrow{\key}+\thisrow{\key_std}, col sep=comma] {./figs/cifar10_miss_all.csv};
\addplot[name path=all_lower,draw=none,forget plot] table[x=k,y expr=\thisrow{\key}-\thisrow{\key_std}, col sep=comma] {./figs/cifar10_miss_all.csv};
\addplot [fill=red!30,opacity=0.7,forget plot] fill between[of=all_upper and all_lower];

\addplot[name path=last_upper,draw=none,forget plot] table[x=k,y expr=\thisrow{\key}+\thisrow{\key_std}, col sep=comma] {./figs/cifar10_miss.csv};
\addplot[name path=last_lower,draw=none,forget plot] table[x=k,y expr=\thisrow{\key}-\thisrow{\key_std}, col sep=comma] {./figs/cifar10_miss.csv};
\addplot [fill=blue!30,opacity=0.5,forget plot] fill between[of=last_upper and last_lower];

\addplot[black, dashed] table[x=k, y=baseline, col sep=comma]{./figs/cifar10_miss.csv};
\addplot[red, mark=triangle*, mark size=3pt] table[x=k, y=\key, col sep=comma]{./figs/cifar10_miss_all.csv};
\addplot[blue, mark=*, mark size=2pt] table[x=k, y=\key, col sep=comma]{./figs/cifar10_miss.csv};
\end{axis}
\node[below=1.1cm of axis.south] {\leg};
\end{tikzpicture}%
\\
  \def\xs{0}%
  \def\ys{0}%
  \def\key{sgd_all}%
  \def\leg{(e) Proposed}%
  \begin{tikzpicture}
\begin{axis}[
name=axis,
xshift=\xs cm,
yshift=\ys cm,
width=0.45\textwidth,
height=120pt,
xmode=log,
xmin=1,
xmax=10000,
ymin=0.145,
ymax=0.19,
xlabel=\# of instances removed,
ylabel=Misclassification rate,
label style={font=\footnotesize},
scaled ticks=false,
tick label style={
	/pgf/number format/fixed,
    /pgf/number format/precision=4,
    font=\scriptsize
},
legend style={font=\fontsize{7}{5}\selectfont},
legend style={at={(0,1.08)},anchor=south,legend columns=-1}
]
\addplot[name path=b_upper,draw=none,forget plot] table[x=k,y expr=\thisrow{baseline}+\thisrow{baseline_std}, col sep=comma] {./figs/cifar10_miss.csv};
\addplot[name path=b_lower,draw=none,forget plot] table[x=k,y expr=\thisrow{baseline}-\thisrow{baseline_std}, col sep=comma] {./figs/cifar10_miss.csv};
\addplot [fill=black!30,opacity=0.3,forget plot] fill between[of=b_upper and b_lower];

\addplot[name path=all_upper,draw=none,forget plot] table[x=k,y expr=\thisrow{\key}+\thisrow{\key_std}, col sep=comma] {./figs/cifar10_miss_all.csv};
\addplot[name path=all_lower,draw=none,forget plot] table[x=k,y expr=\thisrow{\key}-\thisrow{\key_std}, col sep=comma] {./figs/cifar10_miss_all.csv};
\addplot [fill=red!30,opacity=0.7,forget plot] fill between[of=all_upper and all_lower];

\addplot[name path=last_upper,draw=none,forget plot] table[x=k,y expr=\thisrow{\key}+\thisrow{\key_std}, col sep=comma] {./figs/cifar10_miss.csv};
\addplot[name path=last_lower,draw=none,forget plot] table[x=k,y expr=\thisrow{\key}-\thisrow{\key_std}, col sep=comma] {./figs/cifar10_miss.csv};
\addplot [fill=blue!30,opacity=0.5,forget plot] fill between[of=last_upper and last_lower];

\addplot[black, dashed] table[x=k, y=baseline, col sep=comma]{./figs/cifar10_miss.csv};
\addplot[red, mark=triangle*, mark size=3pt] table[x=k, y=\key, col sep=comma]{./figs/cifar10_miss_all.csv};
\addplot[blue, mark=*, mark size=2pt] table[x=k, y=\key, col sep=comma]{./figs/cifar10_miss.csv};
\end{axis}
\node[below=1.1cm of axis.south] {\leg};
\end{tikzpicture}%

  \def\xs{0}%
  \def\ys{0}%
  \def\key{sgd_last}%
  \def\leg{(f) Proposed (Approx.)}%
  \begin{tikzpicture}
\begin{axis}[
name=axis,
xshift=\xs cm,
yshift=\ys cm,
width=0.45\textwidth,
height=120pt,
xmode=log,
xmin=1,
xmax=10000,
ymin=0.145,
ymax=0.19,
xlabel=\# of instances removed,
ylabel=Misclassification rate,
label style={font=\footnotesize},
scaled ticks=false,
tick label style={
	/pgf/number format/fixed,
    /pgf/number format/precision=4,
    font=\scriptsize
},
legend style={font=\fontsize{7}{5}\selectfont},
legend style={at={(0,1.08)},anchor=south,legend columns=-1}
]
\addplot[name path=b_upper,draw=none,forget plot] table[x=k,y expr=\thisrow{baseline}+\thisrow{baseline_std}, col sep=comma] {./figs/cifar10_miss.csv};
\addplot[name path=b_lower,draw=none,forget plot] table[x=k,y expr=\thisrow{baseline}-\thisrow{baseline_std}, col sep=comma] {./figs/cifar10_miss.csv};
\addplot [fill=black!30,opacity=0.3,forget plot] fill between[of=b_upper and b_lower];

\addplot[name path=all_upper,draw=none,forget plot] table[x=k,y expr=\thisrow{\key}+\thisrow{\key_std}, col sep=comma] {./figs/cifar10_miss_all.csv};
\addplot[name path=all_lower,draw=none,forget plot] table[x=k,y expr=\thisrow{\key}-\thisrow{\key_std}, col sep=comma] {./figs/cifar10_miss_all.csv};
\addplot [fill=red!30,opacity=0.7,forget plot] fill between[of=all_upper and all_lower];

\addplot[name path=last_upper,draw=none,forget plot] table[x=k,y expr=\thisrow{\key}+\thisrow{\key_std}, col sep=comma] {./figs/cifar10_miss.csv};
\addplot[name path=last_lower,draw=none,forget plot] table[x=k,y expr=\thisrow{\key}-\thisrow{\key_std}, col sep=comma] {./figs/cifar10_miss.csv};
\addplot [fill=blue!30,opacity=0.5,forget plot] fill between[of=last_upper and last_lower];

\addplot[black, dashed] table[x=k, y=baseline, col sep=comma]{./figs/cifar10_miss.csv};
\addplot[red, mark=triangle*, mark size=3pt] table[x=k, y=\key, col sep=comma]{./figs/cifar10_miss_all.csv};
\addplot[blue, mark=*, mark size=2pt] table[x=k, y=\key, col sep=comma]{./figs/cifar10_miss.csv};
\end{axis}
\node[below=1.1cm of axis.south] {\leg};
\end{tikzpicture}%

\caption{Exhaustive results on CIFAR10: [Thick lines] Average misclassification rates on the test set after data cleansing over 30 experiments. [Shaded Regions] Average $\pm$ standard deviation.}
\label{fig:app_mis_cifar10_exhaustive}
\end{figure}

\begin{figure}[t]
\begin{tikzpicture}
\centering
\begin{axis}[
name=axis1,
axis equal,
width=160pt,
height=160pt,
xmin=0.005,
xmax=0.016,
ymin=0.005,
ymax=0.016,
xlabel style={align=center},
xlabel=Misclassification rate\\ No Removal,
ylabel style={align=center},
ylabel=Misclassification rate\\ Proposed,
label style={font=\footnotesize},
scaled ticks=false,
tick label style={
	/pgf/number format/fixed,
    /pgf/number format/precision=4,
    font=\scriptsize
},
]
\addplot[blue, mark=*, only marks] table[x=baseline, y=sgd_all, col sep=comma]{./figs/mnist_compare_m0100.csv};
\addplot [black, dashed, domain=0.0:1.0, samples=3] {x};
\end{axis}
\begin{axis}[
name=axis2,
xshift=2.7in,
axis equal,
width=160pt,
height=160pt,
xmin=0.14,
xmax=0.20,
ymin=0.14,
ymax=0.20,
xlabel style={align=center},
xlabel=Misclassification rate\\ No Removal,
ylabel style={align=center},
ylabel=Misclassification rate\\ Proposed,
label style={font=\footnotesize},
scaled ticks=false,
tick label style={
	/pgf/number format/fixed,
    /pgf/number format/precision=4,
    font=\scriptsize
},
]
\addplot[blue, mark=*, only marks] table[x=baseline, y=sgd_all, col sep=comma]{./figs/cifar10_compare_m0100.csv};
\addplot [black, dashed, domain=0.0:1.0, samples=3] {x};
\end{axis}
\node [below=1.5cm of axis1.south] {(a) MNIST};
\node [below=1.5cm of axis2.south] {(b) CIFAR10};
\end{tikzpicture}
\caption{Comparison of the misclassification rates before and after the data cleansing with the proposed method. We set the number of removed instances to be 100 both for MNIST and CIFAR10.}
\label{fig:compare}
\end{figure}

\begin{figure}[t]
    \centering
    \subfigure[Autoencoder]{\includegraphics[width=0.49\textwidth]{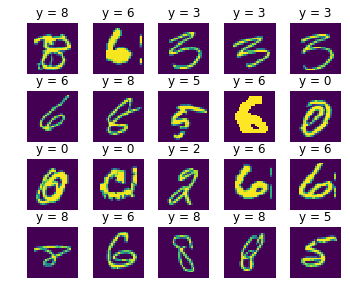}}
    \subfigure[Isolation Forest]{\includegraphics[width=0.49\textwidth]{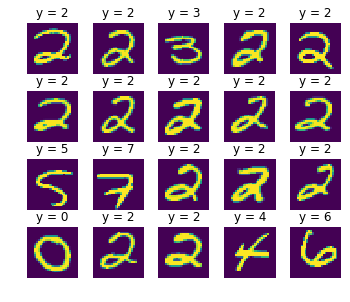}}
    \subfigure[Proposed]{\includegraphics[width=0.49\textwidth]{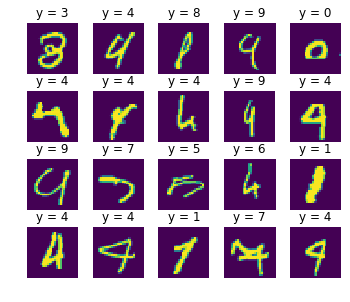}}
    \subfigure[Proposed (Approx.)]{\includegraphics[width=0.49\textwidth]{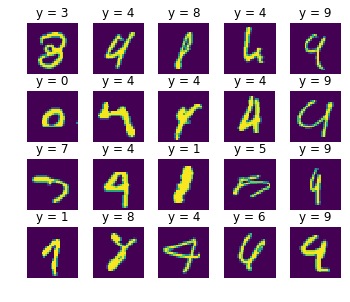}}
    \subfigure[K\&L~\citep{koh2017understanding}]{\includegraphics[width=0.49\textwidth]{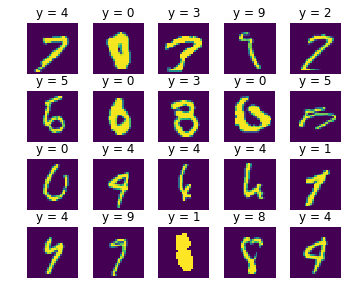}}
    \caption{Examples of found top-20 influential instances in MNIST}
    \label{fig:example_mnist}
\end{figure}

\begin{figure}[t]
    \centering
    \subfigure[Autoencoder]{\includegraphics[width=0.49\textwidth]{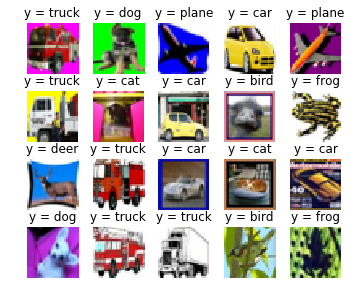}}
    \subfigure[Isolation Forest]{\includegraphics[width=0.49\textwidth]{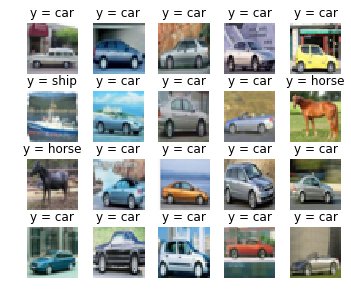}}
    \subfigure[Proposed]{\includegraphics[width=0.49\textwidth]{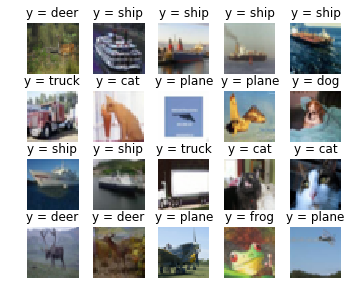}}
    \subfigure[Proposed (Approx.)]{\includegraphics[width=0.49\textwidth]{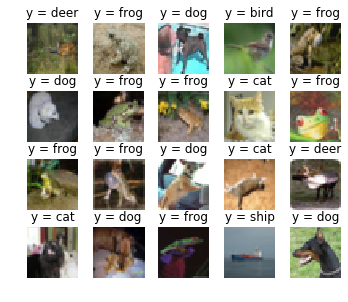}}
    \subfigure[K\&L~\citep{koh2017understanding}]{\includegraphics[width=0.49\textwidth]{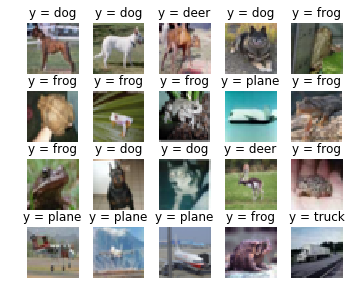}}
    \caption{Examples of found top-20 influential instances in CIFAR10}
    \label{fig:example_cifar10}
\end{figure}

\end{document}